\documentclass[opre,nonblindrev]{informs3n} 

\DoubleSpacedXI 

\usepackage{endnotes}
\let\footnote=\endnote
\let\enotesize=\normalsize
\def\notesname{Endnotes}%
\def\makeenmark{$^{\theenmark}$}
\def\enoteformat{\rightskip0pt\leftskip0pt\parindent=1.75em
  \leavevmode\llap{\theenmark.\enskip}}

\usepackage{natbib}
 \bibpunct[, ]{(}{)}{,}{a}{}{,}%
 \def\bibfont{\small}%

\usepackage{subcaption}
\usepackage{verbatim}

\definecolor{RED}{rgb}{1,0,0}
\definecolor{ORANGE}{rgb}{1,0.5,0}
\definecolor{BLUE}{rgb}{0,0,1}

\newcommand{\indep}{\perp \!\!\! \perp}
\usepackage{pgf, tikz} 
\usetikzlibrary{arrows,automata,fit}
\usetikzlibrary{shapes,snakes}
\newcommand{\xx}{1}
\newcommand{\yy}{1}
\newcommand{\A}{\mathbb{A}}
\newcommand{\B}{\mathbb{B}}
\newcommand{\C}{\mathbb{C}}
\newcommand{\E}{\mathbb{E}}
\newcommand{\Oo}{\mathbb{O}}
\newcommand{\U}{\mathbb{U}}
\newcommand{\Y}{\mathbb{Y}}
\newcommand{\y}{\pmb{y}}
\newcommand{\ex}{\textnormal{E}}
\newcommand{\p}{\textnormal{Pr}}
\newcommand{\Mm}{\mathcal{T}}
\newcommand{\Bb}{\mathcal{M}}
\usepackage[ruled,vlined]{algorithm2e}
\usepackage{booktabs}
\newcommand{\stages}[2]{\tikz{\node[shape=circle,draw,inner sep=1pt,fill=#1,minimum size=0.5cm]{${#2}$};}} 

\TheoremsNumberedThrough     
\ECRepeatTheorems

\EquationsNumberedThrough    

\begin{document}


\RUNAUTHOR{Ballester-Ripoll and Leonelli}

\RUNTITLE{Global Sensitivity Analysis in Probabilistic Graphical Models}

\TITLE{Global Sensitivity Analysis in Probabilistic Graphical Models}

\ARTICLEAUTHORS{%
\AUTHOR{Rafael Ballester-Ripoll}
\AFF{School of Human Sciences \& Technology, IE University, Madrid, Spain.\{\EMAIL{rafael.ballester@ie.edu}\}} 
\AUTHOR{Manuele Leonelli}
\AFF{School of Human Sciences \& Technology, IE University, Madrid, Spain. \{\EMAIL{manuele.leonelli@ie.edu}\}} 
} 

\ABSTRACT{%
We show how to apply Sobol's method of global sensitivity analysis to measure the influence exerted by a set of nodes' evidence on a quantity of interest expressed by a Bayesian network. Our method exploits the network structure so as to transform the problem of Sobol index estimation into that of marginalization inference. This way, we can efficiently compute indices for networks where brute-force or Monte Carlo based estimators for variance-based sensitivity analysis would require millions of costly samples. Moreover, our method gives exact results when exact inference is used, and also supports the case of correlated inputs. The proposed algorithm is inspired by the field of tensor networks, and generalizes earlier tensor sensitivity techniques from the acyclic to the cyclic case. We demonstrate the method on three medium to large Bayesian networks that cover the areas of project risk management and reliability engineering.
}%



\maketitle

%


\section{Introduction}

The assessment of the effect of the choice of inputs on the outputs of a mathematical model is a fundamental but often overlooked step of any real-world analysis. The set of techniques to perform such a critical task is usually referred to under the broad name of \emph{sensitivity analysis}. There is an increasing recognition of the importance of performing sensitivity analysis after the definition of a mathematical model and there are now a wide array of recent reviews of the methods available \citep[e.g.][]{Borgonovo2016,Borgonovo2017,Razavi2021,Saltelli2019}.

Probabilistic graphical models and specifically Bayesian networks (BNs) \citep[e.g.][]{Darwiche2009,Koller2009} are a class of models that are widely used for risk assessment of complex operational systems in a variety of domains \citep[e.g.][]{Bielza2014a,Cai2018,Drury2017,Mclachlan2020}. They provide an efficient and intuitive framework to represent the relationship between a vector of variables of interest by simply using a graph. Sensitivity analysis is increasingly becoming part of any real-world BN modeling \citep[e.g.][]{Hanninen2012,Kleemann2017,Makaba2020} and general-purpose software that implements it is now available \citep{Leonelli2021,samIam}.

Sensitivity analysis in BNs can be broken down into two main steps. First, some parameters of the model are varied and the effect of these variations on output probabilities of interest are investigated. For this purpose, a simple mathematical function, usually termed \emph{sensitivity function}, describes an output probability of interest as a function of the BN parameters \citep{Castillo1997,Coupe2002}.  Second, once parameter variations are identified, their effect is summarized by a distance or divergence measure between the original and the varied distributions underlying the BN, most commonly the Chan-Darwiche distance \citep{Chan2005} or the well-known Kullback-Leibler divergence. 

As noticed by \citet{Rohmer2020} all the above-mentioned methods are local since they most often measure the effect of a parameter variation on an output of interest, while the other parameters are kept fixed. Multi-way sensitivity analyses, where more than one parameter is varied contemporaneously, have been proposed \citep[e.g.][]{Bolt2014,Chan2004,Kjaerulff2000,Leonelli2017,Leonelli2019}. However, these quickly become computationally intractable.

Outside of the realm of BNs, the most common approach is the so-called global sensitivity analysis \citep{Saltelli2000,SRACCGST:08} which provides a ``global'', instead of local, representation of how different factors jointly influence some function of the model's output. General GSA is computationally intensive and remains largely unused for probabilistic graphical models. Here we bridge this gap and contribute a practical computation of Sobol indices~\citep{Sobol:90} for variables associated to vertices of a network. To our knowledge, the only other attempt at introducing global sensitivity methods for BNs is due to \citet{LM:17}. However, their approach suffers of a number of drawbacks: i. it is computationally very expensive and requires a complex Monte Carlo simulation; ii. its implementation is not freely available; iii. it can only compute Sobol indices for one vertex of the network.

Our method takes advantage of the similarities between tensor networks \citep[TNs; see e.g.][]{YL:18} and probabilistic graphical models and is part of an ongoing effort to foster interactions between these two communities \citep{RS:18}.
%
%
Like probabilistic networks, TNs have long had an important role in knowledge discovery and mathematical modeling, and are experiencing a surge in popularity along with the latest trends in machine learning~\citep{SS:16}. Although topologies admitting cycles are much more expressive, much of the extant TN literature focuses on learning acyclic structures from observable multidimensional arrays (tensors). Further, the Sobol indices of an acyclic tensor networks can be computed efficiently, especially in the particular case of the \emph{linear tensor network} or \emph{tensor train}~\citep{Rai:14,ZYOKD:15,BPP:18,BPP:19}. Our contribution is \emph{to make exact Sobol sensitivity analysis possible in the case where the function of interest is given by a Bayesian network}; note that such networks can only be cast as tensor networks by introducing cycles (via moralization).

Importantly, by taking advantage of the underlying structure of the BN, the methods proposed here can formally account for correlated inputs, although some restrictions need to be imposed. Global sensitivity methods usually ignore input correlations which can greatly falsify the outcome of the sensitivity analysis \citep{Do2020}, although some recent proposals attempt to embed correlated inputs \citep[e.g.][]{Chastaing2012,Iooss2019,Sh2020}.

The implementation of our methods is freely available in Python (available at \url{https://github.com/rballester/bnsobol}) contributing to the recent effort of promoting global sensitivity analysis \citep{Douglas2020}.

\subsection{Notation}

We use the following throughout the paper:

\begin{itemize}
	\item The capital $Y$ represents random variables, $y$ their realizations and $\mathbb{Y}$ their sample spaces.
	\item Bold letters for vectors, e.g. $\pmb{y}$, $\pmb{Y}$, and blackboard letters for sets, e.g. $\mathbb{A}$, $\mathbb{B}$; $\mathcal{P}(\mathbb{A})$ is the power set of $\mathbb{A}$.
	\item $[n]$ denotes $\{1, \dots, n\}$.
\item For a subset $\mathbb{A} \subset [n]$, we denote $\pmb{Y}_{\mathbb{A}}=(Y_i)_{i\in \mathbb{A}}$ and $\mathbb{Y}_{\mathbb{A}}=\times_{i\in \mathbb{A}}\mathbb{Y}_i$.
	\item Calligraphic letters are used for graphical models: $\mathcal{B}, \mathcal{M}$ and $\mathcal{T}$ for Bayesian networks, Markov random fields and tensor networks, respectively.
	\item If $\mathcal{T}$ is a graphical model and $\mathbb{A}\subset [n]$, $\mathcal{T}(\y_\A)$ is the model that results after setting $\pmb{Y}_\A = \y_\A$.
	\item A subscript, as in $\mathcal{T}_{\mathbb{A}}$, denotes variable elimination, i.e. removing $\pmb{Y}_\A$ from a model $\mathcal{T}$.
	\item The backslash ($\setminus$) denotes the absence of an element in a set. We use it whenever the set in question is clear from the context; for example, for a network $\mathcal{T}$ on variables $[n]$, $\mathcal{T}_{ \setminus i}$ denotes removing $\pmb{Y}_{[n] \setminus \{i\}}$.
\end{itemize}

\section{Bayesian Networks and Graphical Models}
\label{sec:bn}

BNs are one of the most common statistical methods to investigate the dependence structure for a random vector of interest and to efficiently answer inferential queries. Next we review the theory of BNs and discuss their connection to more general probabilistic graphical models.

\subsection{Bayesian Networks}

Consider a random vector $\pmb{Y} =(Y_1,\dots,Y_n)$ of interest, where for the purposes of this paper we assume $Y_i$ to be taking values in a discrete space $\mathbb{Y}_i$, $i\in[n]$. For three random vectors $\pmb{Y}_{\mathbb{A}}$, $\pmb{Y}_{\mathbb{B}}$ and $\pmb{Y}_{\mathbb{C}}$, where $\mathbb{A}, \mathbb{B}, \mathbb{C} \subset [n]$, we say that $\pmb{Y}_{\mathbb{A}}$ is conditionally independent of $\pmb{Y}_{\mathbb{B}}$ given $\pmb{Y}_{\mathbb{C}}$ and write $\pmb{Y}_{\mathbb{A}}\indep \pmb{Y}_{\mathbb{B}}|\pmb{Y}_{\mathbb{C}}$ if 
\[
\textnormal{Pr}(\pmb{Y}_{\mathbb{A}}=\pmb{y}_{\mathbb{A}}|\pmb{Y}_{\mathbb{B}}=\pmb{y}_{\mathbb{B}},\pmb{Y}_{\mathbb{C}}=\pmb{y}_{\mathbb{C}})=\textnormal{Pr}(\pmb{Y}_{\mathbb{A}}=\pmb{y}_{\mathbb{A}}|\pmb{Y}_{\mathbb{C}}=\pmb{y}_{\mathbb{C}}),
\]
for all $\pmb{y}_{\mathbb{A}}\in\mathbb{Y}_{\mathbb{A}}$, $\pmb{y}_{\mathbb{B}}\in\mathbb{Y}_{\mathbb{B}}$ and $\pmb{y}_{\mathbb{C}}\in\mathbb{Y}_{\mathbb{C}}$. In the following, as a shorthand, we write $\textnormal{Pr}(\pmb{Y}_{\mathbb{A}}=\pmb{y}_{\mathbb{A}}|\pmb{Y}_{\mathbb{B}}=\pmb{y}_{\mathbb{B}})$ as $\textnormal{Pr}(\pmb{y}_{\mathbb{A}}|\pmb{y}_{\mathbb{B}})$, for any sets $\mathbb{A}, \mathbb{B} \subseteq[n]$.

A BN gives a visual representation of conditional independence by means of a directed acyclic graph (DAG). For a review of DAGs refer to Appendix 1. The DAG associated to the BN has vertex set $[n]$, i.e. a vertex is associated to each variable of $\pmb{Y}$ and edges represent dependence as formalized next.

\begin{definition}
A BN $\mathcal{B}$ for a discrete random vector $\pmb{Y}$ consists of:
\begin{itemize}
\item a DAG $G$ with vertex set $[n]$ and edge set $\{(j,i): i\in[n], j\in \mathbb{P}_i\}$, encoding $m-1$ conditional independences $Y_i\indep Y_{\mathbb{ND}_i}\;|\, Y_{\mathbb{P}_i}$, where $\mathbb{P}_i$ and $\mathbb{ND}_i$ are the parent and non-descendant sets of $i$ in $G$;
\item conditional probabilities $\textnormal{Pr}(y_i|\pmb{y}_{\mathbb{P}_i})$, for $y_i\in\mathbb{Y}_i$ and $\pmb{y}_{\mathbb{P}_i}\in\mathbb{Y}_{\mathbb{P}_i}$.
\end{itemize}
We let $\mathbb{V}(\mathcal{B})$ and $\mathbb{E}(\mathcal{B})$ denote the vertex and edge sets, respectively, of the DAG of the BN $\mathcal{B}$.
\end{definition}

Figure \ref{fig:bn} gives an illustration of a DAG over five discrete variables $(Y_1,\dots,Y_5)$. This DAG embeds the conditional independences $Y_2\indep Y_1$, $Y_3\indep Y_2|Y_1$, $Y_4\indep Y_3| Y_1,Y_2$ and $Y_5\indep Y_1,Y_2| Y_3,Y_4$. The definition of the BN is completed by a numerical specification of the probabilities of each variable conditional on the possible values of the parents.

\begin{figure}
\FIGURE
{\centering
\begin{tikzpicture}
\renewcommand{\xx}{2}
\renewcommand{\yy}{1.5}
\node (1) at (0*\xx,0*\yy){\stages{white}{1}};
\node (2) at (0*\xx,1*\yy){\stages{white}{2}};
\node (3) at (1*\xx,0*\yy){\stages{white}{3}};
\node (4) at (1*\xx,1*\yy){\stages{white}{4}};
\node (5) at (2*\xx,0.5*\yy){\stages{white}{5}};
\draw[->, line width = 1.1pt] (1) -- (3);
\draw[->, line width = 1.1pt] (2) -- (4);
\draw[->, line width = 1.1pt] (3) -- (5);
\draw[->, line width = 1.1pt] (4) -- (5);
\draw[->, line width = 1.1pt] (1) -- (4);
\end{tikzpicture}}
{A simple DAG of a BN with vertex set $\{1,2,3,4,5\}$ and edge set $\{(1,3),(1,4),(2,4),(3,5),(4,5)\}$.\label{fig:bn}}
{}
\end{figure}

Importantly, a BN specifies a factorization of the joint probability of $\pmb{Y}$ in terms of simpler conditional distributions of individual variables $Y_i$ conditional on their parents $\pmb{Y}_{\mathbb{P}_i}$, $i\in[n]$, i.e.
\begin{equation}
\label{eq:factorization}
\textnormal{Pr}(\pmb{y}) = \prod_{i=2}^n\textnormal{Pr}(y_i|\pmb{y}_{\mathbb{P}_i})\textnormal{Pr}(y_1),
\end{equation}
for $\pmb{y} = (y_1,\dots,y_n)\in\mathbb{Y}$. For the BN in Figure \ref{fig:bn}, it holds that
\[
\textnormal{Pr}(\pmb{y})=\textnormal{Pr}(y_5|y_4,y_3)\textnormal{Pr}(y_4|y_2,y_1)\textnormal{Pr}(y_3|y_1)\textnormal{Pr}(y_2)\textnormal{Pr}(y_1)
\]

\subsection{Using Bayesian Networks}

The DAG associated to a BN provides an intuitive overview of the relationships existing between variables of interest. However, it does also provide a framework to assess if any generic conditional independence holds for a specific subset of the variables via the so-called \emph{d-separation} criterion \citep[see e.g.][]{Darwiche2009}. Furthermore, the DAG provides a framework for fast exact propagation of probabilities and evidence for the computation of any (conditional) probability involving a specific subset of the variables. 

Both of the above tasks usually require a manipulation of the original DAG, called \emph{moralization}. In a nutshell, moralization adds an undirected edge between any two parents of a vertex which are not joined by an edge (hence the word moralize). This is illustrated for the DAG in Figure \ref{fig:bn}. The vertices 3 and 4 are parents of 5, but there is no edge between them. The same argument holds for the vertices 1 and 2, which are parents of 4. Thus the moralization process adds an undirected edge between 1 and 2 and between 3 and 4. The moralized DAG is reported in Figure \ref{fig:skeleton} (left). The \emph{skeleton} of a DAG consists of the moralized DAG where all directed edges are replaced by undirected ones. The skeleton of the DAG in Figure \ref{fig:bn} is given in Figure \ref{fig:skeleton} (right).

The skeleton is an undirected graph, with the extra property of being decomposable (see Appendix 1), giving the same probability to events $\pmb{y}\in\mathbb{Y}$ as the original BN. Such graphs enable for fast probability computations taking advantage of an alternative factorization of $\textnormal{Pr}(\pmb{Y}=\pmb{y})$ to the one in Equation (\ref{eq:factorization}). Let $\mathbb{C}$ be the set of cliques of the skeleton and $\mathbb{S}$ the set of separators of the skeleton. Then
\begin{equation*}
\label{eq:cliques}
\textnormal{Pr}(\pmb{y})=\frac{\prod_{C \in \mathbb{C}}\textnormal{Pr}(\pmb{y}_{C})}{\prod_{S\in\mathbb{S}}\textnormal{Pr}(\pmb{y}_{S})}
\end{equation*}
For example, the skeleton in Figure \ref{fig:skeleton} has cliques $\{1,2,4\}$, $\{1,3,4\}$ and $\{3,4,5\}$, and separators $\{1,4\}$ and $\{3,4\}$. Therefore, its probabilities factorize as
\[
\textnormal{Pr}(\pmb{y})=\frac{\textnormal{Pr}(y_1,y_2,y_4)\textnormal{Pr}(y_1,y_3,y_4)\textnormal{Pr}(y_3,y_4,y_5)}{\textnormal{Pr}(y_1,y_4)\textnormal{Pr}(y_3,y_4)}
\]
where the probabilities on the right-hand side are easily computed from the conditional probabilities defining the BN: for instance, $\textnormal{Pr}(y_1,y_2,y_4)=\textnormal{Pr}(y_4|y_1,y_2)\textnormal{Pr}(y_2)\textnormal{Pr}(y_1)$.

\begin{figure}
\FIGURE
	{\begin{subfigure}[b]{0.48\columnwidth}
\centering
		\begin{tikzpicture}
\renewcommand{\xx}{2}
\renewcommand{\yy}{1.5}
\node (1) at (0*\xx,0*\yy){\stages{white}{1}};
\node (2) at (0*\xx,1*\yy){\stages{white}{2}};
\node (3) at (1*\xx,0*\yy){\stages{white}{3}};
\node (4) at (1*\xx,1*\yy){\stages{white}{4}};
\node (5) at (2*\xx,0.5*\yy){\stages{white}{5}};
\draw[->, line width = 1.1pt] (1) -- (3);
\draw[line width = 1.1pt] (3) -- (4);
\draw[line width = 1.1pt] (1) -- (2);
\draw[->, line width = 1.1pt] (2) -- (4);
\draw[->, line width = 1.1pt] (3) -- (5);
\draw[->, line width = 1.1pt] (4) -- (5);
\draw[->, line width = 1.1pt] (1) -- (4);
\end{tikzpicture}
	\end{subfigure}
	\hfil
	\begin{subfigure}[b]{0.48\columnwidth}
\centering
\begin{tikzpicture}
\renewcommand{\xx}{2}
\renewcommand{\yy}{1.5}
\node (1) at (0*\xx,0*\yy){\stages{white}{1}};
\node (2) at (0*\xx,1*\yy){\stages{white}{2}};
\node (3) at (1*\xx,0*\yy){\stages{white}{3}};
\node (4) at (1*\xx,1*\yy){\stages{white}{4}};
\node (5) at (2*\xx,0.5*\yy){\stages{white}{5}};
\draw[line width = 1.1pt] (1) -- (3);
\draw[line width = 1.1pt] (3) -- (4);
\draw[line width = 1.1pt] (1) -- (2);
\draw[line width = 1.1pt] (2) -- (4);
\draw[line width = 1.1pt] (3) -- (5);
\draw[line width = 1.1pt] (4) -- (5);
\draw[line width = 1.1pt] (1) -- (4);
\end{tikzpicture}
	\end{subfigure}}
	{Moralization process for the BN in Figure \ref{fig:bn}: moralized DAG (left) and skeleton (right). \label{fig:skeleton}}
{}
\end{figure}

%

\subsection{Markov Random Fields}

Starting from a BN, the moralization process returns an undirected graph which provides an efficient framework to carry out a variety of tasks including d-separation, probability propagation, etc. However, in many applications it is customary to define a probabilistic graphical model based on undirected graphs directly \citep[see e.g.][]{Guo2015,Ma2019,Salemi2019}. Such models are usually called \emph{Markov Random Fields} (MRFs). Here we follow \citet{RS:18} and define them via hypergraphs (see the Appendix for details).

\begin{definition}
\label{def:mrf}
An MRF $\mathcal{M}$ with respect to an hypergraph with vertex set $[n]$ and hyperedges $\mathbb{C}_{\mathcal{M}}$ is defined by
\begin{align*}
\phi_{\mathcal{M}}: \mathbb{Y} \rightarrow& \mathbb{R}_{+}\\
 \phi_{\mathcal{M}}(\pmb{y}) \mapsto &\frac{1}{z}\prod_{C\in\mathcal{P}([n])}\phi_{\mathcal{M}}(\pmb{y}_C),
\end{align*} 
where $\phi_{\mathcal{M}}(\pmb{y}_C)=1$ if $C\not\in\mathbb{C}_{\mathcal{M}}$ and otherwise is a generic function from $\mathbb{Y}_\C$ to $\mathbb{R}_{+}$. The functions $\phi_{\mathcal{M}}(\y_\C)$ are called \emph{potentials} and $z$ is a normalizing constant to impose that $\sum_{\y\in\Y}\phi_\mathcal{M}(\y)=1$. We write $\mathbb{V}(\mathcal{M})=[n]$ and $\mathbb{E}(\mathcal{M})=\mathbb{C}_{\mathcal{M}}$.
\end{definition}

The previous definition is very general but is needed when working with TNs. More often, MRFs are defined via a simple undirected graph, and the function $\phi$ factorizes over the cliques of the graph (for this reason we denoted the hyperedges as $\mathbb{C}_\mathcal{M}$). Conversely, Definition \ref{def:mrf} allows the factorization to work over any subsets of the vertex set. Notice that $\phi_\mathcal{M}(\pmb{y})$ is actually a probability distribution since it is non-negative and sums to one.

As an illustration, consider the undirected graph in Figure \ref{fig:skeleton}, whose cliques are $\{1,2,4\}$, $\{1,3,4\}$, $\{3,4,5\}$. An MRF over this undirected graph can be defined by constructing an hypergraph with the same vertex set and as hyperedges its cliques. Its probability distribution can be written as
\[
\textnormal{Pr}(\pmb{y})=\phi_\mathcal{M}(\y)=\frac{1}{z}\phi_{\mathcal{M}}(y_1,y_2,y_4)\phi_{\mathcal{M}}(y_1,y_3,y_4)\phi_{\mathcal{M}}(y_3,y_4,y_5).
\]

Every BN can be cast as an MRF, as formalized next.

\begin{proposition}
\label{prop:eq}
A BN $\mathcal{B}$ with $\mathbb{V}(\mathcal{B})=[n]$ is equivalent to an MRF $\mathcal{M}$ such that $\mathbb{V}(\mathcal{M})=[n]$ and $\mathbb{E}(\mathcal{M})$ is the set of cliques of the skeleton of the DAG of $\mathcal{B}$.
\end{proposition}

Our new global sensitivity analysis methods take advantage of the transformation of a BN to its skeleton to perform efficient probability computations. For this reason, although our methods will be illustrated for BNs only, they also apply for MRFs as well as chain graphs (which again are usually transformed to an undirected graph).

In particular, an operation required for probability propagation and that our methods will heavily use is that of \emph{marginalization}. Given an MRF $\mathcal{M}$ over the vector $\pmb{Y}$ and a set $\mathbb{A} \subset [n]$ we want to compute the probability $\textnormal{Pr}(\pmb{y}_{\mathbb{A}})$ which can be derived as
\begin{equation*}
\label{eq:marg}
\textnormal{Pr}(\pmb{y}_{\mathbb{A}})=\sum_{\pmb{y}_{\setminus \mathbb{A}}\in\mathbb{Y}_{\setminus \mathbb{A}}} \textnormal{Pr}(\pmb{y}_{\mathbb{A}},\pmb{y}_{\setminus \mathbb{A}})= \frac{1}{z}\sum_{\pmb{y}_{\setminus \mathbb{A}}\in\mathbb{Y}_{\setminus \mathbb{A}}}\prod_{C\in\mathbb{C}_\mathcal{M}}\phi_{\mathcal{M}}(\pmb{y}_{C}).
\end{equation*}

Since the normalizing constant $z$ is often hard to compute, we will simply work with potentials and un-normalized probabilities and call:
\[
\mathcal{M}_{\setminus \mathbb{A}}(\pmb{y}_{\mathbb{A}})=\sum_{\pmb{y}_{\setminus \A}\in\mathbb{Y}_{\setminus \A}}\prod_{C \in \mathbb{C}_{\mathcal{M}}}\phi_{\mathcal{M}}(\pmb{y}_{C}),
\]
so that $\textnormal{Pr}(\pmb{y}_{\mathbb{A}})=\frac{1}{z}\mathcal{M}_{\setminus \mathbb{A}}(\pmb{y}_{\mathbb{A}})$.


\subsection{Tensor Networks}

Since MRFs encode probabilities in a non-normalized way, they are still subject to the non-negativity constrain. \emph{Tensor networks} drop this condition and thus become a graphical representation for general functions~\cite{RS:18}. Therefore, a TN $\mathcal{T}$ is defined as an MRF $\mathcal{M}$, but whose potentials $\phi_\mathcal{T}$ are generic functions to $\mathbb{R}$.
 For instance, the problem of tensor decomposition seeks to find a low-parametric network that encodes, either exactly or approximately, a given real or complex-valued input tensor.

In this paper we do not aim to learn a TN out of data. Instead, we turn the input BN $\mathcal{B}$, which may have been learned from data, into its equivalent MRF as in Proposition \ref{prop:eq}, and then derive a TN $\mathcal{T}$ from $\mathcal{M}$ by simply letting $\mathbb{V}(\mathcal{T})=\mathbb{V}(\mathcal{M})$, $\mathbb{E}(\mathcal{T})=\mathbb{C}_{\mathcal{M}}$ and $\phi_{\mathcal{T}}(\pmb{y}_C)=\phi_{\mathcal{M}}(\pmb{y}_C)$ if $C\in\mathbb{C}_{\mathcal{M}}$ and 1 otherwise. Such a TN is said to be \emph{derived} from $\mathcal{M}$. Given this TN $\mathcal{T}$ we then manipulate it in order to extract relevant sensitivity indices for the original BN $\mathcal{B}$: see Section \ref{sec:method} below.

\section{Sobol Indices} \label{sec:sobol_indices}

The method of~\cite{Sobol:90} is one of the most powerful frameworks for global sensitivity analysis of a general function $f: \Omega \subseteq \mathbb{R}^k \to \mathbb{R}$, whose domain is usually restricted to a $k$-dimensional rectangle $\Omega = \times_{i \in [k]} \Omega_i$. In our case, as formalized in Section \ref{sec:bn}, $\Omega$ will be a discrete Cartesian product $\mathbb{Y}_{\mathbb{E}}$ for a set $\mathbb{E}$ of variables in a BN.

There are multiple kinds of Sobol indices, each with different nuances of interpretation. Next, we review the two most important ones.

\subsection{Variance Components} \label{sec:variance_components}

The \emph{variance component} associated to variable $Y_i$ is defined as
\begin{equation}
S_i := \frac{\mathrm{Var}_i \left[ \textnormal{E}_{\setminus i}[f] \right] }{\mathrm{Var}[f]},
\label{eq:variance_component}
\end{equation}
where $ \textnormal{E}_{\setminus i}$ is the expectation with respect to $\pmb{Y}_{\setminus i}$ and $\mathrm{Var}_i$ is the variance with respect to $Y_i$. This index measures the \emph{additive effect} that is due to $Y_i$ or, more plainly, how the average model output changes as we vary $y_i$.

\subsection{Total Indices} \label{sec:total_indices}

The \emph{total index} of the variable $Y_i$ swaps the order of expectation and variance in Equation~(\ref{eq:variance_component}):
\begin{equation}
S^T_i := \frac{ \textnormal{E}_{\setminus i} \left[ \mathrm{Var}_i[f] \right] }{\mathrm{Var}[f]}.
\label{eq:total_index}
\end{equation}

The $S^T_i$ measures the overall effect due to $Y_i$ or, in plain words, the average variability that results from changing $y_i$ while leaving all other parameters fixed. This effect generalizes the additive effect. When $f$'s inputs are all mutually independent, this implies $S_i \le S^T_i$; we have $S_i = S^T_i$ if and only if $Y_i$ can be separated from $f$ as an additive component, that is, if the original function can be decomposed as $f(\pmb{y}) = g(y_i) + h(\pmb{y}_{\setminus i})$. 

Among other uses, total indices are a popular tool for model simplification: a small $S^T_i$ guarantees that variable $Y_i$ can be safely \emph{frozen} to any value of its domain and removed from the model.

\subsection{Algorithmic Challenges}

Being able to capture global effects, Sobol indices are highly expressive and have become a gold-standard in sensitivity analysis. However, since both the expectation and the variance in Equations~(\ref{eq:variance_component}) and~(\ref{eq:total_index}) are multidimensional integrals (or multi-index summations, in the discrete case), they are notoriously difficult to compute. Estimation for black-box functions is often undertaken via Monte Carlo sampling routines which, due to the Central Limit Theorem, have a slow convergence rate of $O(\sqrt{n})$ in general where $n$ is the number of samples taken. Even when a Quasi-Monte Carlo sampling plan is used, for example via Sobol or Saltelli sequences~\citep{SRACCGST:08}, large sampling budgets (say, with $n$ between $10^4$ and $10^8$) are often needed in practice for each and every Sobol index. Alternatively, more efficient algorithms exist when $f$ is represented as a \emph{surrogate model}, which is the traditional sensitivity analysis approach to approximate expensive function evaluations. This includes \emph{polynomial chaos expansions}~\citep{Sudret:08} and \emph{Gaussian processes}~\citep{MILR:09}, among others. Unfortunately, such surrogate models are often poorly suited to approximate Bayesian networks, which are the target of this paper.



\section{Proposed Method for Global Sensitivity Analysis in BNs} \label{sec:method}

As customary in sensitivity analysis for BNs \citep[see e.g.][]{Gomez2013}, we assume a BN $\mathcal{M}$ and a partition of its vertices into an output node $\mathbb{O}$, evidential nodes $\mathbb{E}$ and chance nodes $\mathbb{U}$, i.e. $[n]= \mathbb{O}\cup \mathbb{E}\cup \mathbb{U}$. In more details:
\begin{itemize}
	\item $Y_\mathbb{O}$ is an \emph{output variable} whose expected value is a quantity of interest.
	\item $\pmb{Y}_\mathbb{E} $ is a vector of \emph{evidential variables} whose influence on $Y_{\mathbb{O}}$ we wish to quantify.
	\item $\pmb{Y}_{\mathbb{U}} $ is a vector of \emph{chance variables} which are unobserved and cannot be fixed. 
\end{itemize}

Formally, given such a network, we wish to study the following function of interest:
\begin{align}
\begin{split} \label{eq:f}
f: \mathbb{Y}_{\mathbb{E}} &\to \mathbb{R} \\
\pmb{y}_{\mathbb{E}} &\mapsto \textnormal{E}_{\mathbb{U}}[Y_{\mathbb{O}} | \pmb{Y}_{\mathbb{E}}=\pmb{y}_{\mathbb{E}}],
\end{split}
\end{align}
where $ \textnormal{E}_{\mathbb{U}}$ refers to the expectation with respect to $\pmb{Y}_{\mathbb{U}}$. Without loss of generality we assume there is a single output variable $Y_{\mathbb{O}}$ (see the remark below). 

Specifically, our goal is to assess the \emph{global effect} of observed evidence $\pmb{Y}_{\mathbb{E}}=\pmb{y}_{\mathbb{E}}$ on the quantity of interest $f(\pmb{y}_{\mathbb{E}})$ from Equation~(\ref{eq:f}). In Sobol's variance-based sense, the effect of an evidential variable $Y_i$ is captured by two indices $S_i$ and $S^T_i$ for every $i \in \mathbb{E}$. As argued earlier, such effects are key to important tasks including model simplification, elicitation, interpretation, factor prioritization, risk analysis, etc.~\citep{SRACCGST:08}. Ours is the first method to address the case in which $f$ is defined by a probabilistic graphical model as a function of evidence in several of its nodes. Note that explicitly evaluating $f(\pmb{y}_{\mathbb{E}})$ requires inference in a BN. This operation, which bears significant computational cost for one sample $\pmb{y}_{\mathbb{E}}$, becomes quickly impractical for thousands or millions of samples, as is typically required by Monte Carlo based methods. Building a surrogate model is rarely viable and accurate enough for general BNs, and the case of correlated input variables is even more challenging. Our algorithms avoid the explicit evaluation of $f$ and, instead, exploit the network's structure to efficiently obtain all variances required in Equations~(\ref{eq:variance_component}) and~(\ref{eq:total_index}).

One must also account for the way evidential variables are distributed: they follow a joint probability mass function $\textnormal{Pr}(\pmb{y}_{\mathbb{E}}) = \mathcal{M}_{\setminus \mathbb{E}}(\pmb{y}_{\mathbb{E}})$ for $\pmb{y}_{\mathbb{E}}\in\mathbb{Y}_{\mathbb{E}}$. If all vertices $\mathbb{E}$ are root nodes (i.e. they have no parents), $\textnormal{Pr}$ is separable, which is an important particular case in the sensitivity analysis literature; in fact, many applications of the method of Sobol assume uncorrelated inputs. For example, when input values can be set at will (as in an experimental design), it is customary to assume independent, uniformly distributed inputs in a rectangle $\Omega \subset \mathbb{R}^{|\mathbb{E}|}$. The method we propose supports correlated inputs as well by means of deriving the network $\mathcal{M}_{\setminus \mathbb{E}}$, which we do via marginalization. While marginalizing the offspring of $\mathbb{E}$ is straightforward (no operations are required), marginalizing its ancestors is more computationally intensive. Therefore, our method is most useful when the evidential nodes are located in a relatively upstream position in the graph.

\paragraph*{Remark.}

Note that our formulation is flexible enough to allow for arbitrary output functions to be considered, including those that represent the \emph{utility} of a set of nodes. For example, suppose we initially have a network with nodes $\mathbb{E} \cup \mathbb{U}$ and a utility function $g: \mathbb{Y}_{\mathbb{D}} \to \mathbb{Y}_\mathbb{O}$ where $\mathbb{D} \subseteq \mathbb{E} \cup \mathbb{U}$ is a subset of the network's variables, and assume we want to study $g$'s output as a function of variables $\mathbb{E}$. To reduce this task to Equation~(\ref{eq:f}), we encode $g$ into a new node $\mathbb{O}$ that we add to the graph as in \citet{Luque2010}. The domain of $Y_{\mathbb{O}}$ is $\mathbb{Y}_\mathbb{O}$, its parents are $\pmb{Y}_\mathbb{D}$, and its conditional probability table $\textnormal{Pr}$ is built so as to reflect $g$:
\begin{equation*}
\textnormal{Pr}(y_{\mathbb{O}}|\pmb{y}_{\mathbb{D}}) = 
\begin{cases}
1 \mbox{ if } y_{\mathbb{O}} = g(\pmb{y}_{\mathbb{D}}) \\
0 \mbox{ otherwise}.
\end{cases}
\end{equation*}
This function encoding strategy is demonstrated in one of our experiments in Section~\ref{sec:simulations}.

\subsection{Network Arithmetics} \label{sec:network_arithmetics}

The proposed algorithm relies on a few building blocks for TN manipulation, which are introduced next.

\begin{definition}
\label{def:square}
	Given a TN $\mathcal{T}$ with vertex set $[n]$ and $\mathbb{A} \subseteq [n]$, the square of $\mathcal{T}$ with respect to $\mathbb{A}$ is a TN $\mathcal{T}^2$ such that $(\mathcal{T}^2)_{\setminus \mathbb{A}}(\pmb{y}_{\mathbb{A}}) = (\mathcal{T}_{\setminus \mathbb{A}}(\pmb{y}_{\mathbb{A}}))^2$ for any $ \pmb{y}_{\mathbb{A}}\in\mathbb{Y}_{\mathbb{A}}$. Since the squaring commutes with the marginalization, we may simply write $\mathcal{T}^2_{\setminus \mathbb{A}}$.
\end{definition}

\begin{proposition}
\label{prop:square}
Let $\mathbb{B} = [n] \setminus \mathbb{A}$ and $\mathcal{T}$ be a TN. The square $\mathcal{T}^2$ w.r.t. $\mathbb{A}$ can be constructed as follows:
\begin{itemize}
\item $\mathbb{V}(\mathcal{T}^2)=\mathbb{A}\cup\mathbb{B}\cup\widehat{\mathbb{B}}$ where $\mathbb{A}\cap \widehat{\mathbb{B}}= \emptyset$, $\mathbb{B}\cap \widehat{\mathbb{B}}= \emptyset$ and there exists a bijection $\eta: \widehat{\mathbb{B}} \rightarrow \mathbb{B}$;
\item $\mathbb{E}(\mathcal{T}^2)=\mathbb{E}(\mathcal{T})\cup \mathbb{E}'$ where $\mathbb{E}'$ contains the hyperedge $\{\mathbb{A}\cap \mathbb{D}\} \cup \eta(\mathbb{D}\setminus \mathbb{A})$ for every $\mathbb{D}\in\mathbb{C}$;
\item if $\mathbb{D}\subseteq \mathbb{A}\cup\mathbb{B}$ then $\phi_{\mathcal{T}^2}(\pmb{y}_\mathbb{D})=\phi_{\mathcal{T}}(\pmb{y}_\mathbb{D})$, otherwise $\mathbb{D}\subseteq \mathbb{A}\cup\widehat{\mathbb{B}}$ and $\phi_{\mathcal{T}^2}(\pmb{y}_\mathbb{D})= \phi_{\mathcal{T}}(\pmb{y}_{C\cap\eta(\widehat{\mathbb{B}})},\pmb{y}_{\mathbb{D}\cap \mathbb{A}})$.
\end{itemize}
\end{proposition}


\proof{Proof.}
For any $\pmb{y}_{\mathbb{A}}\in\mathbb{Y}_{\mathbb{A}}$ we have
\begin{align}
(\mathcal{T}_{\setminus \mathbb{A}}(\pmb{y}_{\mathbb{A}}))^2 &=\left( \sum_{\pmb{y}_{\mathbb{B}}\in\mathbb{Y}_{\mathbb{B}}}\prod_{\mathbb{D}\in\mathcal{P}(\mathbb{V}(\mathcal{T}))}\phi_{\mathcal{T}}(\pmb{y}_{\mathbb{A}\cap\mathbb{D}},\pmb{y}_{\mathbb{B}\cap\mathbb{D}})\right)^2\nonumber \\
&= \sum_{\pmb{y}_{\mathbb{B}},\pmb{y}_{\mathbb{B}}'\in\mathbb{Y}_{\mathbb{B}}}\prod_{\mathbb{D}\in\mathcal{P}(\mathbb{V}(\mathcal{T}))}\phi_{\mathcal{T}}(\pmb{y}_{\mathbb{A}\cap\mathbb{D}},\pmb{y}_{\mathbb{B}\cap\mathbb{D}})\phi_{\mathcal{T}}(\pmb{y}_{\mathbb{A}\cap\mathbb{D}},\pmb{y}_{\mathbb{B}\cap\mathbb{D}}')\nonumber\\
&= \sum_{\pmb{y}_{\mathbb{B}},\pmb{y}_{\mathbb{B}}'\in\mathbb{Y}_{\mathbb{B}}}\prod_{\mathbb{D}\in\mathcal{P}(\mathbb{V}(\mathcal{T}))}\phi_{\mathcal{T}^2}(\pmb{y}_{\mathbb{A}\cap\mathbb{D}},\pmb{y}_{\mathbb{B}\cap\mathbb{D}})\phi_{\mathcal{T}^2}(\pmb{y}_{\mathbb{A}\cap\mathbb{D}},\pmb{y}_{\mathbb{B}\cap\eta(\widehat{\mathbb{D}})}')\nonumber\\
& = \sum_{\pmb{y}_{\mathbb{B}\cup \widehat{\mathbb{B}}}\in \mathbb{Y}_{\mathbb{B}\cup \widehat{\mathbb{B}}}} \prod_{\mathbb{D}\in\mathcal{P}(\mathbb{V}(\mathcal{T}^2))} \phi_{\mathcal{T}^2}(\pmb{y}_{\mathbb{A}\cap \mathbb{D}}, \pmb{y}_{(\mathbb{B}\cup \widehat{\mathbb{B}})\cap \mathbb{D}}) \label{eq:proof1}
\end{align}
where the last equality holds since $\mathbb{E}(\mathcal{T}^2)$ does not include any element with vertices in both $\mathbb{B}$ and $\widehat{\mathbb{B}}$. The result easily follows by noticing that the r.h.s. of Equation (\ref{eq:proof1}) is exactly $\mathcal{T}^2_{\setminus \mathbb{A}}(\pmb{y}_{\mathbb{A}})$. \hfill \Halmos
\endproof

Figure \ref{fig:square} illustrates the squaring procedure of Proposition \ref{prop:square}. A TN $\mathcal{T}$ has hypergraph with vertex set $[5]$ and hyperedges $\{1,2,4\}$, $\{1,3,4\}$ and $\{3,4,5\}$. This can be represented by the undirected graph in Figure \ref{fig:square} (left). Suppose that we are interested in the squared TN $\mathcal{T}^2$ w.r.t. $\{1,2\}$, which, using Proposition \ref{prop:square}, can be constructed as the one in Figure \ref{fig:square} (right). The idea is related to a known algorithm to compute the dot product between two tensor networks $\mathcal{T}$ and $\mathcal{T}'$ by contracting all nodes to yield a scalar $\left< \mathcal{T}, \mathcal{T}' \right>$~\citep{LC:17}. Our version sets $\mathcal{T} = \mathcal{T}'$ and skips the node contraction part.

\begin{figure}
\FIGURE{
\centering
\begin{tikzpicture}
\renewcommand{\xx}{2}
\renewcommand{\yy}{1.5}
\node (1) at (0*\xx,0*\yy){\stages{white}{1}};
\node (2) at (0*\xx,1*\yy){\stages{white}{2}};
\node (3) at (1*\xx,0*\yy){\stages{white}{3}};
\node (4) at (1*\xx,1*\yy){\stages{white}{4}};
\node (5) at (2*\xx,0.5*\yy){\stages{white}{5}};
\draw[line width = 1.1pt] (1) -- (3);
\draw[line width = 1.1pt] (3) -- (4);
\draw[line width = 1.1pt] (1) -- (2);
\draw[line width = 1.1pt] (2) -- (4);
\draw[line width = 1.1pt] (3) -- (5);
\draw[line width = 1.1pt] (4) -- (5);
\draw[line width = 1.1pt] (1) -- (4);
\end{tikzpicture}
\hspace{2cm}
\begin{tikzpicture}
\renewcommand{\xx}{2}
\renewcommand{\yy}{1.5}
\node (1) at (0*\xx,0*\yy){\stages{gray}{1}};
\node (2) at (0*\xx,1*\yy){\stages{gray}{2}};
\node (3) at (1*\xx,0*\yy){\stages{white}{3}};
\node (4) at (1*\xx,1*\yy){\stages{white}{4}};
\node (5) at (2*\xx,0.5*\yy){\stages{white}{5}};
\node (3b) at (-1*\xx,0*\yy){\stages{white}{\widehat{3}}};
\node (4b) at (-1*\xx,1*\yy){\stages{white}{\widehat{4}}};
\node (5b) at (-2*\xx,0.5*\yy){\stages{white}{\widehat{5}}};
\draw[line width = 1.1pt] (1) -- (3);
\draw[line width = 1.1pt] (3) -- (4);
\draw[line width = 1.1pt] (1) -- (2);
\draw[line width = 1.1pt] (2) -- (4);
\draw[line width = 1.1pt] (3) -- (5);
\draw[line width = 1.1pt] (4) -- (5);
\draw[line width = 1.1pt] (1) -- (4);
\draw[line width = 1.1pt] (1) -- (3b);
\draw[line width = 1.1pt] (2) -- (4b);
\draw[line width = 1.1pt] (3b) -- (4b);
\draw[line width = 1.1pt] (3b) -- (5b);
\draw[line width = 1.1pt] (4b) -- (5b);
\draw[line width = 1.1pt] (1) -- (4b);
\end{tikzpicture}
}
{Squaring process for the tensor network $\mathcal{T}$ (left) with respect to $\{1,2\}$ (right). \label{fig:square}}
{}
\end{figure}

\begin{definition}
	Let $\mathcal{T}$ and $\mathcal{T}'$ be two TNs for $\pmb{Y}$ whose hypergraphs have the same vertex set $[n]$. The quotient of $\mathcal{T}$ over $\mathcal{T}'$ denoted as $\mathcal{T}/\mathcal{T}'$ is such that $(\mathcal{T}/\mathcal{T}')(\pmb{y}) = \mathcal{T}(\pmb{y}) / \mathcal{T'}(\pmb{y})$ for any $\pmb{y} \in \mathbb{Y}$.
\end{definition}

\begin{proposition} \label{prop:network_quotient}
	The quotient $\mathcal{T}/\mathcal{T}'$ is a TN whose hypergraph has vertex set $[n]$ and hyperedges $\mathbb{E}(\mathcal{T})\cup\mathbb{E}(\mathcal{T}')$, and potentials $\phi_{\mathcal{T}/\mathcal{T}'}(\pmb{y})=\phi_{\mathcal{T}}(\pmb{y})/\phi_{\mathcal{T}'}(\pmb{y})$ for any $\pmb{y} \in \mathbb{Y}$.

\end{proposition}

\proof{Proof.}
	For any $\pmb{y} \in \mathbb{Y}$ we have
\[
		\frac{\mathcal{T}(\pmb{y})}{\mathcal{T}'(\pmb{y})} = \frac{\prod_{\mathbb{A}\in\mathcal{P}(\mathbb{V}(\mathcal{T}))} \phi_{\mathcal{T}}(\pmb{y}_\mathbb{A})}{\prod_{\mathbb{A}\in\mathcal{P}(\mathbb{V}(\mathcal{T'}))} \phi_{\mathcal{T}'}(\pmb{y}_\mathbb{A})} = \prod_{\mathbb{A}\in\mathcal{P}(\mathbb{V}(\mathcal{T}))}\frac{\phi_{\mathcal{T}}(\pmb{y}_{\mathbb{A}})}{\phi_{\mathcal{T}'}(\pmb{y}_{\mathbb{A}})} = \prod_{{\mathbb{A}\in\mathcal{P}(\mathbb{V}(\mathcal{T}/\mathcal{T}'))}} \phi_{\mathcal{T}/\mathcal{T}'}(\pmb{y}_{\mathbb{A}}) = (\mathcal{T}/\mathcal{T}')(\pmb{y}) \hfill\Halmos
\]
	
\endproof

\begin{proposition} \label{prop:quotient_projection}
	Let $\mathbb{A}\subset [n]$. Then, $(\mathcal{T}/\mathcal{T}_{\mathbb{A}}')_{\mathbb{A}} = \mathcal{T}_{\mathbb{A}} / \mathcal{T}_{\mathbb{A}}'$. In addition, $(\mathcal{T}/\mathcal{T}_{\mathbb{A}}')^2_{\mathbb{A}} = (\mathcal{T}_{\mathbb{A}})^2 / (\mathcal{T}_{\mathbb{A}}')^2$, where the square is w.r.t. $\setminus\mathbb{A}$.
\end{proposition}

\proof{Proof.}
	For the first statement and any $\pmb{y}_{\setminus\mathbb{A}} = \mathbb{Y}_{\setminus\mathbb{A}}$,
\begin{align*}
		(\mathcal{T}/\mathcal{T}_{\mathbb{A}}')_{\mathbb{A}}(\pmb{y}_{\setminus\mathbb{A}}) & = \sum_{\pmb{y}_{\mathbb{A}}\in\mathbb{Y}_{\mathbb{A}}} (\mathcal{T}(\pmb{y}_{\mathbb{A}}, \pmb{y}_{\setminus \mathbb{A}}) / \mathcal{T}_{\mathbb{A}}'(\pmb{y}_{\setminus
\mathbb{A}}))\\
& = \left( \sum_{\pmb{y}_{\mathbb{A}}\in\mathbb{Y}_{\mathbb{A}}} \mathcal{T}(\pmb{y}_{\mathbb{A}}, \pmb{y}_{\setminus \mathbb{A}}) \right) / \mathcal{T}_{\mathbb{A}}'(\pmb{y}_{\setminus \mathbb{A}}) = \mathcal{T}_{\mathbb{A}}(\pmb{y}_{\setminus\mathbb{A}}) / \mathcal{T}_{\mathbb{A}}'(\pmb{y}_{\setminus\mathbb{A}}).
\end{align*}
	The second statement follows immediately from the above and Definition \ref{def:square}. \hfill\Halmos
\endproof

\subsection{Representing the Target Function}
The quotient operation is also useful to build a TN representation of the function of interest $f$.

\begin{definition} \label{def:function_tn}
Let $\mathcal{B}$ be our starting point BN, i.e. with a vertex set that decomposes into evidential, chance and output nodes: $[n] = \E \cup \U \cup \Oo$. Let $\mathcal{M}$ be the MRF derived from $\mathcal{B}$ (Prop.~\ref{prop:eq}). The \emph{function TN} of $\mathcal{M}$ is the tensor network that results by i. dropping the non-negativity condition on $\mathcal{M}$; ii. adding the potential $\phi(y_\Oo)= y_\Oo$.
\end{definition}

\begin{proposition}\label{prop:f}
Let $\mathcal{T}$ be the function TN of the MRF $\mathcal{M}$ (Def.~\ref{def:function_tn}). Then, the function of interest $f$ (Equation~\ref{eq:f}) equals $(\mathcal{T} / \mathcal{M}_{\mathbb{U}\cup\mathbb{O}})_{\mathbb{U}\cup\mathbb{O}} = \mathcal{T}_{\mathbb{U}\cup\mathbb{O}} / \mathcal{M}_{\mathbb{U}\cup\mathbb{O}}$. 
\end{proposition}

\proof{Proof.}
	For any $\y_\mathbb{E} \in \mathbb{Y}_{\mathbb{E}}$ we have
	\begin{align*}
		f(\pmb{y}_{\mathbb{E}}) &= \ex_{\mathbb{U}\cup\mathbb{O}}[ Y_{\mathbb{O}} | \pmb{Y}_{\mathbb{E}} = \pmb{y}_{\mathbb{E}}] = \sum_{\pmb{y}_{\mathbb{O}\cup\mathbb{U}}\in\mathbb{Y}_{\mathbb{O}\cup\mathbb{U}}} y_{\mathbb{O}} \cdot \p(\pmb{y}_{\mathbb{E}},\pmb{y}_{\mathbb{U}\cup\mathbb{O}}) / \p(\pmb{y}_{\mathbb{E}})\\
& =\sum_{\pmb{y}_{\mathbb{O}\cup\mathbb{U}}\in\mathbb{Y}_{\mathbb{O}\cup\mathbb{U}}} y_{\mathbb{O}} \cdot \mathcal{M}(\pmb{y}_{\mathbb{E}},\pmb{y}_{\mathbb{U}\cup\mathbb{O}}) / \mathcal{M}_{\mathbb{U}\cup\mathbb{O}}(\pmb{y}_{\mathbb{E}}) 
		= \sum_{\pmb{y}_{\mathbb{O}\cup\mathbb{U}}\in\mathbb{Y}_{\mathbb{O}\cup\mathbb{U}}}  \mathcal{T}(\pmb{y}_{\mathbb{E}},\pmb{y}_{\mathbb{U}\cup\mathbb{O}}) / \mathcal{M}_{\mathbb{U}\cup\mathbb{O}}(\pmb{y}_{\mathbb{E}}) = (\mathcal{T} / \mathcal{M}_{\mathbb{U}\cup\mathbb{O}})_{\mathbb{U}\cup\mathbb{O}}(\pmb{y}_{\mathbb{E}}) \hfill\Halmos
	\end{align*}
\endproof
\begin{proposition} \label{prop:expected_value}
	Let $\mathcal{T}$ be the TN derived from a MRF $\mathcal{M}$. For any partition $\{\mathbb{A}, \mathbb{B}, \mathbb{D}\}$ of $[n]$ we have $\ex_{\mathbb{A}}[\mathcal{T}_{\mathbb{B}}] = (\mathcal{T}_{\mathbb{B}} \cdot \mathcal{M}_{\mathbb{B}} / \mathcal{M}_{\mathbb{A} \cup \mathbb{B}})_{\mathbb{A}}$.
\end{proposition}

\proof{Proof.}
	Note that $\ex_{\mathbb{A}}(\mathcal{T}(\pmb{y}_{\setminus \mathbb{A}}))=\sum_{\pmb{y}_{\mathbb{A}\in\mathbb{Y}_{\mathbb{A}}}}\mathcal{T}(\pmb{y}_{\mathbb{A}},\pmb{y}_{\setminus \mathbb{A}})\cdot\p(\y_\A| \pmb{Y}_{\setminus \A} = \y_{\setminus \A})$ and $\p(\y_\mathbb{D})=\Bb_{\A\cup\B}(\y_\mathbb{D})$ for any $\y_{\mathbb{D}}\in\Y_{\mathbb{D}}$. Therefore 
\begin{align*}
\ex_\A(\Mm_\B(\y_{\mathbb{D}}))&=\sum_{\y_\A\in\Y_\A}\Mm_\B(\y_\A,\y_{\mathbb{D}})\cdot\p(\y_\A|\pmb{Y}_\mathbb{D}=\y_\mathbb{D})=\sum_{\y_\A\in\Y_\A}\Mm_\B(\y_\A,\y_{\mathbb{D}})\cdot\p(\y_\A,\y_\mathbb{D})/\p(\y_{\mathbb{D}}) \\
&= \sum_{\y_\A\in\Y_\A} (\Mm_\B(\y_\A,\y_{\mathbb{D}})\cdot \Bb_\B(\y_\A,\y_{\mathbb{D}})/\Bb_{\A\cup\B}(\y_{\mathbb{D}}))= (\Mm_\B\cdot\Bb_\B/\Bb_{\A\cup\B})_\A(\y_{\mathbb{D}}). \hfill\Halmos
\end{align*}
\endproof

\begin{proposition} \label{prop:expected_value_f}
	$\ex[f] = \mathcal{T}_{[n]}$.
\end{proposition}

\proof{Proof.}
Combining Propositions.~\ref{prop:quotient_projection} to \ref{prop:expected_value} and using the fact that $\mathcal{M}_{[n]} = 1$, we have 
\[
\ex[f] = \ex_{\E}[(\mathcal{T}/\mathcal{M}_{\setminus\E})_{\setminus \E}] = ((\mathcal{T}/\mathcal{M}_{\setminus \E})_{\setminus \E} \cdot \mathcal{M}_{\setminus \E} / \mathcal{M}_{\setminus \E \cup \E})_{\E}= (\mathcal{T}_{\setminus \E}/\mathcal{M}_{\setminus \E} \cdot \mathcal{M}_{\setminus \E})_{\E} = (\mathcal{T}_{\setminus \E})_{\E} = \mathcal{T}_{[n]}. \hfill\Halmos
\]
\endproof

Once equipped with the tools described above, we are ready to compute the Sobol indices as defined in Section~\ref{sec:sobol_indices}.

\subsection{Computing the Global Variance of $f$}
\label{sec:global_variance}

Both types of Sobol indices (Equations~\ref{eq:variance_component} and~\ref{eq:total_index}) require computing $\mathrm{Var}[f]$ in the denominator. We do so via the formula $\mathrm{Var}[f] = \ex[f^2] - \ex[f]^2$ and deriving each of these terms as follows:
\begin{itemize}
	\item First term: \begin{align}\ex[f^2]&= \ex_{\E}[(\mathcal{T}/\mathcal{M}_{\setminus \E})^2_{\setminus \E}] \mbox{ (using Prop.~\ref{prop:f})}\nonumber \\
	&= ((\mathcal{T} / \mathcal{M}_{\setminus \E})^2_{\setminus \E} \cdot \mathcal{M}_{\setminus \E} / \mathcal{M}_{\setminus \E \cup \E})_{\E} \mbox{ (using Prop.~\ref{prop:expected_value})} \nonumber\\
	&= (\mathcal{T}^2_{\setminus \E} / \mathcal{M}_{\setminus\E} / \mathcal{M}_{\setminus \E \cup \E})_{\E} \mbox{ (using Prop.~\ref{prop:quotient_projection})}\nonumber \\
	&= (\mathcal{T}^2_{\setminus\E}/\mathcal{M}_{\setminus \E})_{\E} 
 \label{eq:global1} \\
	&= (\mathcal{T}^2/\mathcal{M}_{\setminus \E})_{\E \cup \setminus \E \cup \widehat{\setminus \E}} \mbox{, (where the square of } \mathcal{T} \mbox{ is taken w.r.t. } \E\mbox{)}. \label{eq:global2}
	\end{align}
	Note the computational advantage of Equation~(\ref{eq:global2}) over Equation~(\ref{eq:global1}): the latter requires marginalization over variables $\setminus \E$, then $\E$, while the former marginalizes once over all variables, which gives marginalization heuristics full freedom to find the best order.
	\item Second term: $\ex^2[f] = (\mathcal{T}_{[n]})^2$ (using Prop.~\ref{prop:expected_value_f}).
\end{itemize}

\subsection{Computing the Variance Components} \label{sec:computing_variance_component}

We obtain index $S_i$ via Equation~(\ref{eq:variance_component}), where the numerator is found via the identity
\begin{equation}
\label{eqq:identity}
	\mathrm{Var}_i[\ex_{\setminus i}[f]] = \ex_i[E^2_{\setminus i}[f]] - \ex^2_i[\ex_{\setminus i}[f]]
\end{equation}
and by observing that:
\begin{itemize}
	\item The first term is equal to
	\begin{align*} \ex_i[\ex^2_{\setminus i}[f]] &= \ex_i[\ex^2_{\setminus i}[(\mathcal{T}/\mathcal{M}_{\setminus \E})_{\setminus\E}]] \mbox{ (using Prop.~\ref{prop:f})} \\
	&= \ex_i[((\mathcal{T}/\mathcal{M}_{\setminus \E})_{\setminus \E} \cdot \mathcal{M}_{\setminus \E} / \mathcal{M}_{\setminus \E \cup \setminus i})_{\setminus i}^2] \mbox{ (using Prop.~\ref{prop:expected_value})} \\
	&= \ex_i[(\mathcal{T}_{\setminus \E}/\mathcal{M}_{\setminus \E} \cdot \mathcal{M}_{\setminus \E} / \mathcal{M}_{\setminus \E \cup \setminus i})_{\setminus i}^2]\\ 
	&= \ex_i[(\mathcal{T}_{\setminus \E}/\mathcal{M}_{\setminus i})^2_{\setminus i}] \mbox{ (since } \setminus \E \cup \setminus i = \setminus i)\\
	&= ((\mathcal{T}_{\setminus \E}/\mathcal{M}_{\setminus i})^2_{\setminus i} \cdot \mathcal{M}_{\setminus i} / \mathcal{M}_{\setminus i \cup \{i\}})_i \mbox{ (using Prop.~\ref{prop:expected_value})} \\
	&= (\mathcal{T}^2_{\setminus i} / \mathcal{M}_{\setminus i})_i.
	\end{align*}
	Note that, since $\mathcal{T}_{\setminus i}$ is a one-variable network, $\mathcal{T}^2_{\setminus i}$ is most efficiently computed as $(\mathcal{T}_{\setminus i})^2$.
	\item For the second term, $\ex^2_i[\ex_{\setminus i}[f]] = (\ex_i[\ex_{\setminus i}[f]])^2 = \ex^2[f] = (\mathcal{T}_{[n]})^2$ (using Prop.~\ref{prop:expected_value_f}).
\end{itemize}

The denominator is simply the model's global variance $\mathrm{Var}[f]$ (Section~\ref{sec:global_variance}).

\subsection{Computing the Total Indices}

To obtain index $S^T_i$, we use the following alternative definition:
\begin{equation}
S^T_i = 1-\frac{\mathrm{Var}_{\setminus i} \left[ \ex_i[f] \right] }{\mathrm{Var}[f]},
\label{eq:total_index2}
\end{equation}
which is equivalent~\citep{SRACCGST:08} to the one given in Equation~(\ref{eq:total_index}). To obtain the numerator, we use the identity in Equation~(\ref{eqq:identity})
and observe that:

\begin{itemize}
	\item For the first term, analogously to Section~\ref{sec:computing_variance_component}, substituting $i$ for $\setminus i$ and vice versa we arrive at $\ex_{\setminus i}[\ex^2_i[f]] = (\mathcal{T}^2_i / \mathcal{M}_{\setminus \E \cup \{i\}})_{\setminus i}$, where the square is taken w.r.t. $\E \setminus i$.
	\item For the second term, $\ex^2_{\setminus i}[\ex_i[f]] = (\ex_{\setminus i}[\ex_i[f]])^2 = \ex^2[f] = (\mathcal{T}_{[n]})^2$ (using Prop.~\ref{prop:expected_value_f}).
\end{itemize}

\subsection{Summarized Algorithm}

We put together all steps in Algorithm~\ref{alg:main} so as to minimize the cost of intermediate calculations. Note that our algorithm can also obtain Sobol indices for any vector of variables $\pmb{\alpha}$ (the so-called higher-order indices), since all mathematical arguments above hold true when replacing a single variable $i$ by a vector of them $\pmb{\alpha}$.
\begin{algorithm}[h!]
	\SetAlgoLined
	\KwIn{Bayesian network $\mathcal{B}$ with variables $\Oo\cup\U\cup\E$, variable of interest $i \in \E$, domain $\mathbb{Y}_\Oo$ of the output variable $Y_\Oo$}
	\KwResult{Variance component $S_i$ and total index $S^T_i$ for variable $Y_i$}
	
	$\blacktriangleright$ Derive the MRF $\mathcal{M}$ that is equivalent to $\mathcal{B}$ (Prop.~\ref{prop:eq})
	
	$\blacktriangleright$ Construct the function TN $\mathcal{T}$ of $\mathcal{M}$ (Def.~\ref{def:function_tn})
	
	$\blacktriangleright$ Find $\mathcal{J} := \mathcal{M}_{\setminus \E}$ via marginalization
	
	$\blacktriangleright$ Compute $\mathcal{T}_i, \mathcal{J}_i, \mathcal{T}_{\setminus i}, \mathcal{J}_{\setminus i}, \mathcal{T}_{[n]}$ via marginalization
	
	$\blacktriangleright$ Compute $\mathcal{T}^2, \mathcal{T}^2_i$, and $\mathcal{T}^2_{\setminus i}$ via network squaring (Prop.~\ref{prop:square})
	
	$\blacktriangleright$ Compute $\mathcal{T}^2 / \mathcal{J}, \mathcal{T}^2_i / \mathcal{J}_i$, and $\mathcal{T}^2_{\setminus i} / \mathcal{J}_{\setminus i}$ via network quotient (Prop.~\ref{prop:network_quotient})
	
	$\blacktriangleright$ $V := (\mathcal{T}^2/\mathcal{J})_{\E \cup \setminus \E \cup \widehat{\setminus \E}} - (\mathcal{T}_{[n]})^2$
	
	$\blacktriangleright$ $S_i := ((\mathcal{T}^2_{\setminus i} / \mathcal{J}_{\setminus i})_i - (\mathcal{T}_{[n]})^2) / V$
	
	$\blacktriangleright$ $S^T_i := 1 - \left( (\mathcal{T}^2_i / \mathcal{J}_i)_{\setminus i} - (\mathcal{T}_{[n]})^2 \right) / V$
	
	\Return $S_i, S^T_i$
	\caption{Compute first-order Sobol indices for a given BN}
	\label{alg:main}
\end{algorithm}

\section{Illustrations} \label{sec:simulations}

In order to assess the validity of our algorithms, we now consider two networks based on simulated data and one real-world network. In all cases, we use exact marginalization via the \emph{minimal weight} heuristic for variable ordering~\citep{Kjaerulff:90}: at each step, we marginalize the node such that the product of its neighbors' cardinalities is minimal.

\subsection{Concrete Resistance}

As a first example, we consider a 24-node Bayesian network due to \cite{CK:03} designed to assess the damage of reinforced concrete structures of buildings. There is one output variable (the damage of a reinforced concrete beam), 16 observable variables influencing the damage of reinforced concrete structures and seven intermediate unobservable variables that define some partial states of the structure. The list of variables and their definition is reported in Table \ref{table:resistance} and their cause-effect relationships are summarized by the BN in Figure~\ref{fig:concrete}. The original variables in \citet{CK:03} are continuous and their model is defined as a Gaussian BN \citep[e.g.][]{Darwiche2009}. From their definition we derived a discrete BN as follows: (i) we simulate 1000000 observations from their Gaussian BN; (ii) each variable is independently discretized into three categories (low, medium, high) using the equal frequency method \citep{Nojavan2017}; (iii) retaining the same DAG as in \citet{CK:03} the conditional probabilities for the discrete data are learned. All these steps are carried out with the \texttt{bnlearn} package \citep{Scutari2010}.

\begin{table}
	\centering
	\footnotesize
\TABLE
{Variable labels and names for the concrete resistance network.\label{table:resistance}}{
\setlength{\tabcolsep}{20pt}
\renewcommand{\arraystretch}{0.7}
	\begin{tabular}{ll}
\toprule
Variable & Definition\\
\midrule
$O$ & Damage assessment\\
$U_7$ & Cracking state\\
$U_6$ & Cracking state in shear domain\\
$U_5$ & Steel corrosion\\
$U_4$ & Cracking state in flexure domain\\
$U_3$ & Shrinkage cracking\\
$U_2$ & Worst cracking in flexure domain\\
$U_1$ & Corrosion state\\
$E_{16}$ & Weakness of the beam\\
$E_{15}$ & Deflection of the beam\\
$E_{14}$ & Position of the worst shear crack\\
$E_{13}$ & Breadth of the worst shear crack\\
$E_{12}$ & Position of the worst flexure crack\\
$E_{11}$ & Breadth of the worst flexure crack\\
$E_{10}$ & Length of the worst flexure cracks\\
$E_9$ & Cover\\
$E_8$ & Structure age\\
$E_7$ & Humidity\\
$E_6$ & pH value in the air\\
$E_5$ & Content of chlorine in the air\\
$E_4$ & Number of shear cracks\\
$E_3$ & Number of flexure cracks\\
$E_2$ & Shrinkage\\
$E_1$ & Corrosion\\
\bottomrule
\end{tabular}}
{}
\end{table}

The 16 evidential variables are the roots of the DAG (that is, without parents), for a total of $3^{16} \approx 4.31 \cdot 10^7$ grid points in the domain of the function of interest $f$. The resulting concrete resistances are given by output variable $O$, whose values we map to $\mathbb{R}$ as follows: $\mathrm{low} \equiv 0, \mathrm{medium} \equiv 1$, and $\mathrm{high} \equiv 2$. Table~\ref{tab:concrete} and Figure~\ref{fig:concrete} summarize the global sensitivity indices. Overall, computing all 16 variance components and total indices took $1.12$s and $1.79$s, respectively. Each index requires four marginalizations in our implementation, for a total of $16 \cdot 2 \cdot 4 = 128$ such operations that took 2.24s globally. Based on the average marginalization time, a brute-force computation evaluating all entries of $f$ would have taken well above $10^5$ seconds, and even a Monte Carlo approximation using $10^6$ samples would take ca. 18'000s.

The results suggest that the most critical variables for the output are $E_{12}$ (position of the worst flexure crack), $E_{13}$ (breadth of the worst shear crack) and $E_{14}$ (position of the worst shear crack). Notice that the significant differences between variance components and total indices indicate that much of those effects are due to variable interactions, not to single variables alone, therefore highlighting the need of global sensitivity analyses to uncover the actual relationships between input and output variables.

\begin{table}
	\centering
	\footnotesize
\TABLE
{Sobol indices and computation times for the concrete resistance network.
	\label{tab:concrete}}{
\setlength{\tabcolsep}{16pt}
\renewcommand{\arraystretch}{0.7}
	\begin{tabular}{lrrrr}
\toprule
Var. $i$ &   $S_i$ &  Time ($S_i$) &  $S^T_i$ &  Time ($S^T_i$) \\
\midrule
      E1 & 0.05464 &       0.06325 &  0.06335 &         0.09594 \\
      E2 & 0.03081 &       0.05967 &  0.03569 &         0.10103 \\
      E3 & 0.06339 &       0.06109 &  0.07750 &         0.09831 \\
      E4 & 0.04156 &       0.05868 &  0.05470 &         0.09775 \\
      E5 & 0.04837 &       0.06028 &  0.06217 &         0.09880 \\
      E6 & 0.06281 &       0.06016 &  0.07860 &         0.09927 \\
      E7 & 0.01783 &       0.05837 &  0.02503 &         0.09913 \\
      E8 & 0.02344 &       0.05937 &  0.03393 &         0.09868 \\
      E9 & 0.02357 &       0.05922 &  0.03412 &         0.09612 \\
   E10 & 0.05574 &       0.06039 &  0.07461 &         0.09657 \\
     E11 & 0.05560 &       0.05826 &  0.07452 &         0.09817 \\
     E12 & 0.12943 &       0.05956 &  0.15069 &         0.09695 \\
     E13 & 0.14802 &       0.05955 &  0.17370 &         0.09784 \\
     E14 & 0.08480 &       0.05865 &  0.10512 &         0.10081 \\
     E15 & 0.07363 &       0.05868 &  0.07535 &         0.10183 \\
     E16 & 0.01338 &       0.05963 &  0.01388 &         0.09872 \\
\bottomrule
\end{tabular}

}
{}
\end{table}

\begin{figure*}\centering
\FIGURE{\includegraphics[width=0.6\columnwidth]{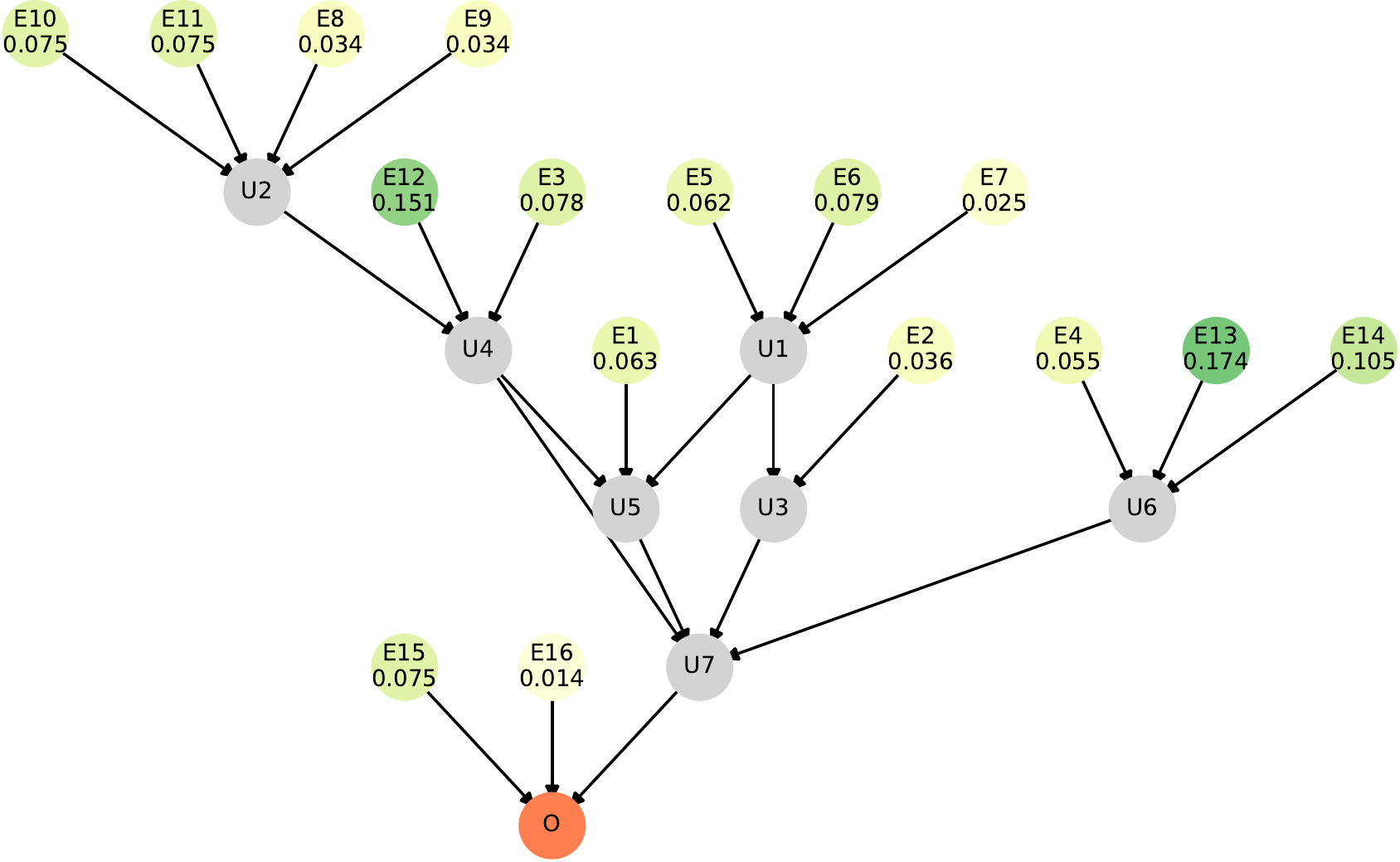}}
	{Total indices for the concrete resistance network: chance nodes are shown in gray, output node in orange, while evidential nodes are color-coded according to their respective total indices.
	\label{fig:concrete}}
{}
\end{figure*}

\subsection{Synthetic Network with Dependent Inputs}

As a second simulated example, we created a 40-node Gaussian Bayesian network with three output nodes, 15 evidential nodes, and 22 chance nodes. Both the structure of the DAG (reported in Figure~\ref{fig:synthetic}) and the conditional distributions of the nodes where chosen a priori so as to best illustrate the methodology. In this case, however, the evidential variables exhibit some dependence. A discrete BN is obtained using the same routine as in the previous example. The output is a new additional node $O(O_1, O_2, O_3) := O_1 + O_2 + O_3$, where $\mathrm{low} \equiv 0, \mathrm{medium} \equiv 1$ and $\mathrm{high} \equiv 2$. Since there are 15 evidential variables, the domain of the function of interest $f$ contains $3^{15} \approx 1.43 \cdot 10^7$ grid points. Table~\ref{tab:synthetic} (right) and Figure~\ref{fig:synthetic} report the global sensitivity indices. Note that $S_i > S^T_i$ for several variables; this is a well-known consequence of the non-independence of the function's inputs. Overall, computing all 15 variance components and total indices took $5.88$s and $11.48$s, respectively. Although, the computation time is slightly larger than in the previous example, it is still feasible to compute the global sensitivity indices in the much more challenging case of dependent inputs.

To illustrate the consequences of overlooking dependence in the inputs, we created the same discrete BN with 15 inputs, but where the dependence between the evidential variables is deleted. The global sensitivity indices for the independent network are reported on the left side of Table~\ref{tab:synthetic}. We can notice two things. First, $S_i^T \ge S_i$ for all variables, due to missed dependence between inputs. Second, evidential variables that are important for the output through the dependence with other inputs (for instance $E_2$ or $E_5$) have now a much smaller effect on the output and their relevance would be overlooked by considering the inputs independent.

\begin{table}
	\centering
	\footnotesize
\TABLE
{Sobol indices and computation times for the synthetic network (with both dependent and independent inputs).
	\label{tab:synthetic}}{
\setlength{\tabcolsep}{6pt}
\renewcommand{\arraystretch}{0.7}
	\begin{tabular}{lrrrr|rrrr}
\toprule
& \multicolumn{4}{c|}{Independent inputs} & \multicolumn{4}{c}{Dependent inputs} \\
Var. $i$ &   $S_i$ &  Time ($S_i$) &  $S^T_i$ &  Time ($S^T_i$) &  $S_i$ &  Time ($S_i$) &  $S^T_i$ &  Time ($S^T_i$) \\
\midrule
      E1 & 0.00812 &       0.38456 &  0.00832 &         0.66675 & 0.05757 &        0.36064 &   0.00437 &          0.63704 \\
     E10 & 0.00617 &       0.35793 &  0.00658 &         0.67984 & 0.00277 &        0.35715 &   0.00311 &          0.65014 \\
     E11 & 0.03099 &       0.35240 &  0.03118 &         0.65987 & 0.01198 &        0.35403 &   0.01237 &          0.64012 \\
     E12 & 0.06793 &       0.34724 &  0.06883 &         0.74234 & 0.04408 &        0.35532 &   0.02202 &          0.71773 \\
     E13 & 0.00471 &       0.35308 &  0.00497 &         0.71207 & 0.02735 &        0.35807 &   0.00482 &          0.72696 \\
     E14 & 0.00336 &       0.35799 &  0.00337 &         0.69925 & 0.00121 &        0.36493 &   0.00124 &          0.65699 \\
     E15 & 0.03169 &       0.36250 &  0.03193 &         0.60682 & 0.01366 &        0.35283 &   0.01395 &          0.61534 \\
      E2 & 0.04169 &       0.35959 &  0.04200 &         0.64520 & 0.15364 &        0.35414 &   0.01426 &          0.64304 \\
      E3 & 0.02196 &       0.35772 &  0.02223 &         0.64319 & 0.10733 &        0.35941 &   0.01452 &          0.63806 \\
      E4 & 0.05410 &       0.35593 &  0.05445 &         0.64833 & 0.13931 &        0.35352 &   0.03504 &          0.65013 \\
      E5 & 0.08062 &       0.35327 &  0.08140 &         0.58844 & 0.27517 &        0.35414 &   0.03071 &          0.60475 \\
      E6 & 0.34252 &       0.35654 &  0.34322 &         0.59323 & 0.37965 &        0.35600 &   0.17480 &          0.61408 \\
      E7 & 0.11156 &       0.35384 &  0.11220 &         0.64536 & 0.16543 &        0.41148 &   0.05473 &          0.64713 \\
      E8 & 0.08040 &       0.35663 &  0.08076 &         0.59313 & 0.13329 &        0.35635 &   0.02463 &          0.61735 \\
      E9 & 0.11108 &       0.35002 &  0.11224 &         0.64339 & 0.17694 &        0.35846 &   0.06924 &          0.64184 \\
\bottomrule
\end{tabular}
}
{}
\end{table}

\begin{figure*}\centering
\FIGURE{
	\includegraphics[width=0.6\columnwidth]{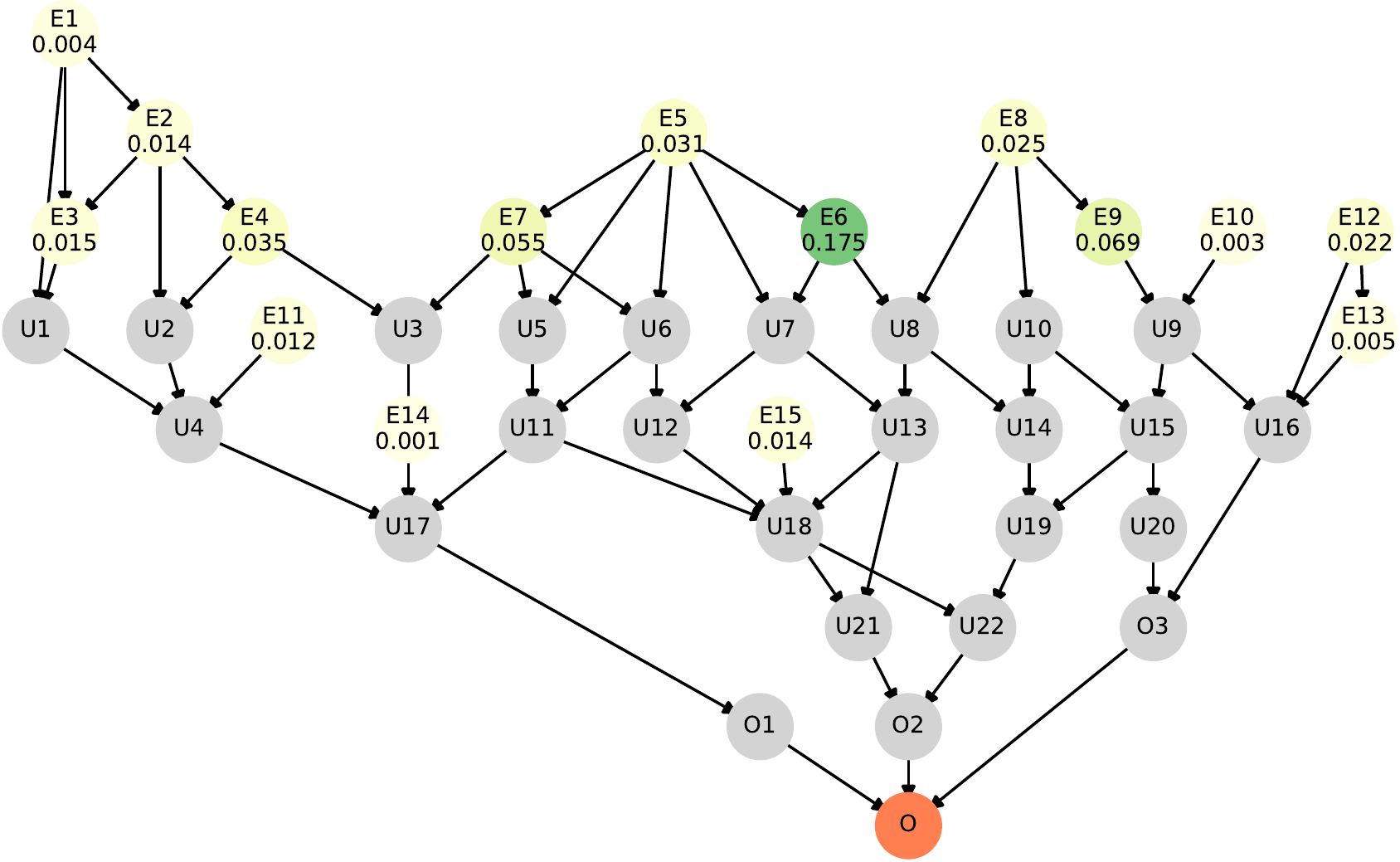}}
	{Total indices for the synthetic networks with dependent inputs: chance nodes are shown in gray, output node in orange, while evidential nodes are color-coded according to their respective total indices.
	\label{fig:synthetic}}
{}
\end{figure*}

\subsection{Risk Assessment in Project Complexity}

Our last example is a BN due to~\cite{QQDK:16} representing
critical risks specific to a construction project and adapted from
an existing model proposed by \citet{Eybpoosh2011}. \citet{QQDK:16} considers eight project complexity elements which can be chosen at the starting of the project as well as four project objectives: timeliness; cost; quality; and market share. All these objectives are henceforth presented as negative counterparts in order to align these to the notion of risks and are assumed to be equally important. There are 14 additional variables that are not observed at the start of the project (chance variables). Table \ref{table:project} reports a full list of the variables, which are all defined as binary. Although the eight project complexity elements are considered as decision variables in \citet{QQDK:16}, giving an influence diagram representation of the problem \citep[e.g.][]{Darwiche2009}, here they are considered as evidential random variables and are assigned a priori an equal probability to their outcomes.

\begin{table}
	\centering
	\footnotesize
\TABLE
{Variable labels and names for the project complexity network.\label{table:project}}{
\setlength{\tabcolsep}{12pt}
\renewcommand{\arraystretch}{0.7}
	\begin{tabular}{ll}
\toprule
Variable & Definition\\
\midrule
$E_1$ & Lack of experience with the involved team\\
$E_2$ & Use of innovative technology\\
$E_3$ & Lack of experience with technology \\
$E_4$ & Strict quality requirements\\
$E_5$ & Multiple contracts \\
$E_6$ & Multiple stakeholder and variety of perspectives\\
$E_7$ & Political instability\\
$E_8$ & Susceptibility to natural disasters\\
$U_1$ & Contractors' lack of experience  \\
$U_2$ & Suppliers default \\
$U_3$ &  Delays in design and regulatory approval\\
$U_4$ & Contract related problems\\
$U_5$ & Economic issues in country \\
$U_6$ & Major design changes \\
$U_7$ & Delays in obtaining raw material \\
$U_8$ & Non-availability of local resources \\
$U_9$ & Unexpected events\\
$U_{10}$ & Increase in raw material prices\\
$U_{11}$ & Change in project specifications\\
$U_{12}$ & Conflicts with project stakeholders\\
$U_{13}$ & Decrease in productivity\\
$U_{14}$ & Delays/interruptions \\
$O_{1}$ & Decrease in quality of work \\
$O_{2}$ & Low market share/reputational issues \\
$O_{3}$ & Time overrun \\
$O_{4}$ & Cost overrrun \\
\bottomrule
\end{tabular}}
{}
\end{table}

Table~\ref{tab:project_management} and Figure~\ref{fig:project_management} report a summary of the results. Overall, computing all 8 variance components and total indices took $1.06$s and $2.00$s, respectively, thus highlighting the computational efficiency of our algorithms. Clearly, the lack of experience with the involved team ($E_1$) has the biggest impact on the possible difficulties involved with the project, with Sobol indices more than twice larger than for all other variables. Other complexity elements that are relevant for the output are (in order of importance) strict quality requirements ($E_4$), lack of experience with technology ($E_3$) and political instability ($E_7$).

\begin{table}
	\centering
	\footnotesize
\TABLE
{Sobol indices and computation times for the project complexity network.\label{tab:project_management}}{
\setlength{\tabcolsep}{12pt}
\renewcommand{\arraystretch}{0.7}
		\begin{tabular}{lrrrr}
\toprule
Var. $i$ &   $S_i$ &  Time ($S_i$) &  $S^T_i$ &  Time ($S^T_i$) \\
\midrule
      E1 & 0.48218 &       0.12326 &  0.48513 &         0.22200 \\
      E2 & 0.00502 &       0.12252 &  0.02151 &         0.21638 \\
      E3 & 0.12845 &       0.12056 &  0.13723 &         0.21051 \\
      E4 & 0.16822 &       0.12094 &  0.18604 &         0.21015 \\
      E5 & 0.00193 &       0.11998 &  0.00466 &         0.21473 \\
      E6 & 0.00093 &       0.12192 &  0.00224 &         0.22305 \\
      E7 & 0.09408 &       0.12048 &  0.10671 &         0.22147 \\
      E8 & 0.08309 &       0.12162 &  0.09678 &         0.22114 \\
\bottomrule
\end{tabular}

}
{}
\end{table}

\begin{figure*}\centering
\FIGURE{
	\includegraphics[width=0.6\columnwidth]{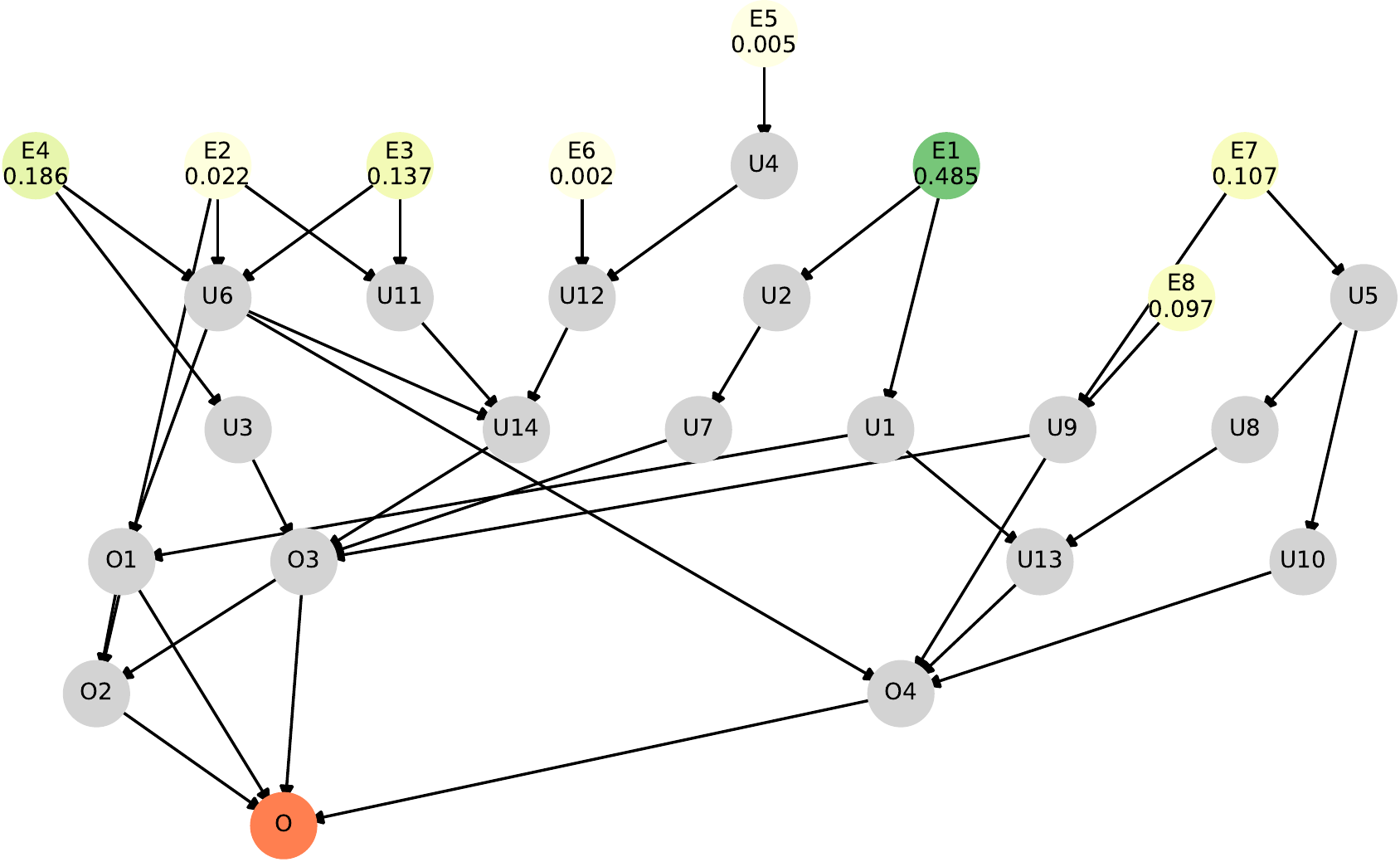}}
	{Total indices for the synthetic networks with dependent inputs: chance nodes and initial output nodes are shown in gray, final output node in orange, while evidential nodes are color-coded according to their respective total indices.
	\label{fig:project_management}}
{}
\end{figure*}

\section{Conclusions}

We have addressed efficient Sobol sensitivity analysis of Bayesian networks, more specifically when the output is the expected value of one of its nodes as a function of a set of other nodes. Our algorithms compute the two most important Sobol indices (variance components and total indices) by first transforming the input BN into a TN that captures the target function. Afterwards, via various manipulations of the TN, we compute each required variance using only a few inference queries (marginalization) over the modified TN. We showed this procedure to scale well: it is able to handle networks where both brute-force and Monte Carlo estimation would be impossible, as they would require millions of inference queries over the BN. In addition, our method is exact, provided that exact inference is affordable on the modified TN.

Our method is a generalization of previous SA algorithms for acyclic tensor networks such as the Tucker or tensor train decompositions. Our version works on cyclic networks, which arise naturally (due to moralization) when casting a general Bayesian network into a TN. The key enabling concepts are those of network squaring and quotient, thanks to which we can manipulate the function of interest so as to produce networks whose marginalization yields the desired statistics.

As a venue of future research, we will explore the feasibility of computing more sophisticated variance-based sensitivity indices from BN-defined functions. These potentially include the mean dimension, the Shapley values, the dimension distribution, and other more recent metrics that are based on the method of Sobol.

%
 \begin{APPENDICES}
\section{Elements of Graph Theory}

\subsection{Graphs}
A graph $G = (\mathbb{V}(G), \mathbb{E}(G))$ is defined by two sets $\mathbb{V}(G)$ and $\mathbb{E}(G)$, where $\mathbb{V}(G)$ is a finite set of vertices, in this paper $[n]$, and $\mathbb{E}(G)$ is a set of edges consisting of a subset of all possible ordered pairs of distinct vertices. The edges of a graph can be directed on undirected, depending on whether or not the order of the involved vertices matters.

Let $G = (\mathbb{V}(G), \mathbb{E}(G))$ be a graph. When $(i,j)\in\mathbb{E}(G)$ and $(j,i)\not\in\mathbb{E}(G)$, we say there is a directed edge from $i$ to $j$. Conversely, if $(i,j), (j,i)\in \mathbb{E}(G)$, we say there is an undirected edge between $i$ and $j$. A graph which all edges are directed is called a directed graph, a graph in which all edges are undirected is called undirected graph, and if both types of edges are present the graph is called mixed. 

A path from vertex $i$ to vertex $j$ is an ordered sequence of vertices $(i_1,\dots,i_r)$ starting in $i=i_1$ and ending in $j=i_r$ such that there is an edge $(i_k,i_{k+1})$ in $G$, $k=1,\dots,r-1$. The length of the path is $r-1$. A path is said to be closed if it has the same starting and ending vertices, that is $i=j$.

\subsection{Directed Graphs}
If $(i,j)\in \mathbb{V}(G)$, $i$ is said to be a parent of $j$ and $j$ is said to be a child of $i$. The set of all parents of a node $i$ is denoted by $\mathbb{P}_i$. The descendant set of a node $i$ is made of all vertices $j$ such that there is a path from $i$ to $j$, and denoted by $\mathbb{D}_i$. The set of non-descendants of a vertex $i$ is $\mathbb{ND}_i=\mathbb{V}(G)\setminus\mathbb{D}_i$. 

A cycle is a closed path in a directed path. A directed acyclic graph is a directed graph with no cycles. The mixed graph obtained by first joining every pair of vertices with a common child in a directed graph is called a moral graph. The undirected graph which substitutes every directed edge in a moral graph with an undirected edge is called a skeleton.

\subsection{Undirected Graphs}
A set of vertices $\mathbb{A}\subseteq \mathbb{V}(G)$ is said to be complete if there are edges between every pair of nodes in $\mathbb{A}$. A complete set of nodes $\mathbb{C}$ is called a clique if it is maximal, that is, it is not a proper subset of another complete set. A loop is a closed path in a directed graph. A chord is and edge between two vertices in a loop that is not contained in that loop. An undirected graph is said to be triangulated if every loop of length four or more has at least one chord. The skeleton of a directed acyclic graph is triangulated.

Let $(\mathbb{C}_1,\dots, \mathbb{C}_k)$ be an ordered sequence of all cliques of an undirected graph and $S_i = C_i \cap \{C_1\cup \cdots \cup C_{i-1}\}$ be the separator. The sequence $(\mathbb{C}_1,\dots, \mathbb{C}_k)$ is said to respect the running intersection property if $S_i$ is a subset of at least one element in $\{C_1,\dots, C_{i-1}\}$, for all $i=1,\dots,k$. An undirected graph has at least one ordered sequence of cliques respecting the running intersection property if and only if it is triangulated. Triangulated graphs, which respect the running intersection property, have great computational advantages, and this is the reason why a directed acyclic graph is usually transformed into its skeleton. We will not give details here about this \citep[which can be found in][]{Darwiche2009}.

\subsection{Hypergraphs}
A hypergraph is a generalization of a graph in which an edge can join any number of vertices. Conversely, in ordinary graphs an edge can only join two vertices. Formally, an hypergraph $H=(\mathbb{V}(H),\mathbb{E}(H))$ where $\mathbb{V}(H)$ is a set of vertices and $\mathbb{E}(H)$ is a set of non-empty subsets of $\mathbb{V}(H)$ called hyperedges. Therefore, $\mathbb{E}(H)\subset \mathcal{P}(\mathbb{V}(H))\setminus \emptyset$.

 \end{APPENDICES}
%
%




\bibliographystyle{informs2014} 
\bibliography{references} 

\begin{thebibliography}{54}
\providecommand{\natexlab}[1]{#1}
\providecommand{\url}[1]{\texttt{#1}}
\providecommand{\urlprefix}{URL }

\bibitem[{Ballester-Ripoll et~al.(2018)Ballester-Ripoll, Paredes,
  \protect\BIBand{} Pajarola}]{BPP:18}
Ballester-Ripoll R, Paredes EG, Pajarola R (2018) Tensor algorithms for
  advanced sensitivity metrics. \emph{SIAM/ASA Journal on Uncertainty
  Quantification} 6(3):1172--1197.

\bibitem[{Ballester-Ripoll et~al.(2019)Ballester-Ripoll, Paredes,
  \protect\BIBand{} Pajarola}]{BPP:19}
Ballester-Ripoll R, Paredes EG, Pajarola R (2019) Sobol tensor trains for
  global sensitivity analysis. \emph{Reliability Engineering \& System Safety}
  183:311--322.

\bibitem[{Bielza \protect\BIBand{} Larra\~{n}aga(2014)}]{Bielza2014a}
Bielza C, Larra\~{n}aga P (2014) Bayesian networks in neuroscience: a survey.
  \emph{Frontiers in Computational Neuroscience} 8:131.

\bibitem[{Bolt \protect\BIBand{} Renooij(2014)}]{Bolt2014}
Bolt JH, Renooij S (2014) Local sensitivity of {B}ayesian networks to multiple
  simultaneous parameter shifts. \emph{European Workshop on Probabilistic
  Graphical Models}, 65--80 (Springer).

\bibitem[{Borgonovo(2017)}]{Borgonovo2017}
Borgonovo E (2017) \emph{Sensitivity analysis: an introduction for the
  management scientist} (Springer).

\bibitem[{Borgonovo \protect\BIBand{} Plischke(2016)}]{Borgonovo2016}
Borgonovo E, Plischke E (2016) Sensitivity analysis: a review of recent
  advances. \emph{European Journal of Operational Research} 248(3):869--887.

\bibitem[{Cai et~al.(2018)Cai, Kong, Liu, Lin, Yuan, Xu, \protect\BIBand{}
  Ji}]{Cai2018}
Cai B, Kong X, Liu Y, Lin J, Yuan X, Xu H, Ji R (2018) Application of
  {B}ayesian networks in reliability evaluation. \emph{IEEE Transactions on
  Industrial Informatics} 15(4):2146--2157.

\bibitem[{Castillo et~al.(1997)Castillo, Guti{\'e}rrez, \protect\BIBand{}
  Hadi}]{Castillo1997}
Castillo E, Guti{\'e}rrez JM, Hadi AS (1997) Sensitivity analysis in discrete
  {B}ayesian networks. \emph{IEEE Transactions on Systems, Man, and
  Cybernetics-Part A: Systems and Humans} 27(4):412--423.

\bibitem[{Castillo \protect\BIBand{} Kj{\ae}rulff(2003)}]{CK:03}
Castillo E, Kj{\ae}rulff U (2003) Sensitivity analysis in {G}aussian {B}ayesian
  networks using a symbolic-numerical technique. \emph{Reliability Engineering
  \& System Safety} 79(2):139--148.

\bibitem[{Chan \protect\BIBand{} Darwiche(2004)}]{Chan2004}
Chan H, Darwiche A (2004) Sensitivity analysis in {B}ayesian networks: from
  single to multiple parameters. \emph{Proocedings of the 16th Conference on
  Uncertainty in Artificial Intelligence}, 317--325.

\bibitem[{Chan \protect\BIBand{} Darwiche(2005)}]{Chan2005}
Chan H, Darwiche A (2005) A distance measure for bounding probabilistic belief
  change. \emph{International Journal of Approximate Reasoning} 38:149--174.

\bibitem[{Chastaing et~al.(2012)Chastaing, Gamboa, \protect\BIBand{}
  Prieur}]{Chastaing2012}
Chastaing G, Gamboa F, Prieur C (2012) Generalized {H}oeffding-{S}obol
  decomposition for dependent variables-application to sensitivity analysis.
  \emph{Electronic Journal of Statistics} 6:2420--2448.

\bibitem[{Coup{\'e} \protect\BIBand{} Van Der~Gaag(2002)}]{Coupe2002}
Coup{\'e} VMH, Van Der~Gaag LC (2002) Properties of sensitivity analysis of
  {B}ayesian belief networks. \emph{Annals of Mathematics and Artificial
  Intelligence} 36(4):323--356.

\bibitem[{Darwiche(2009)}]{Darwiche2009}
Darwiche A (2009) \emph{Modeling and reasoning with {B}ayesian networks}
  (Cambridge University Press).

\bibitem[{Darwiche(2010)}]{samIam}
Darwiche A (2010) \emph{samIam: Sensitivity analysis, modeling, inference and
  more}. \urlprefix\url{http://reasoning.cs.ucla.edu/samiam/}.

\bibitem[{Do \protect\BIBand{} Razavi(2020)}]{Do2020}
Do NC, Razavi S (2020) Correlation effects? {A} major but often neglected
  component in sensitivity and uncertainty analysis. \emph{Water Resources
  Research} 56(3):e2019WR025436.

\bibitem[{Douglas-Smith et~al.(2020)Douglas-Smith, Iwanaga, Croke,
  \protect\BIBand{} Jakeman}]{Douglas2020}
Douglas-Smith D, Iwanaga T, Croke BF, Jakeman AJ (2020) Certain trends in
  uncertainty and sensitivity analysis: an overview of software tools and
  techniques. \emph{Environmental Modelling \& Software} 124:104588.

\bibitem[{Drury et~al.(2017)Drury, Valverde-Rebaza, Moura, \protect\BIBand{}
  de~Andrade~Lopes}]{Drury2017}
Drury B, Valverde-Rebaza J, Moura MF, de~Andrade~Lopes A (2017) A survey of the
  applications of {B}ayesian networks in agriculture. \emph{Engineering
  Applications of Artificial Intelligence} 65:29--42.

\bibitem[{Eybpoosh et~al.(2011)Eybpoosh, Dikmen, \protect\BIBand{}
  Talat~Birgonul}]{Eybpoosh2011}
Eybpoosh M, Dikmen I, Talat~Birgonul M (2011) Identification of risk paths in
  international construction projects using structural equation modeling.
  \emph{Journal of Construction Engineering and Management} 137(12):1164--1175.

\bibitem[{G{\'o}mez-Villegas et~al.(2013)G{\'o}mez-Villegas, Main,
  \protect\BIBand{} Susi}]{Gomez2013}
G{\'o}mez-Villegas MA, Main P, Susi R (2013) he effect of block parameter
  perturbations in {G}aussian {B}ayesian networks: Sensitivity and robustness.
  \emph{Information Sciences} 222:439--458.

\bibitem[{Guo et~al.(2015)Guo, Cheng, Levina, Michailidis, \protect\BIBand{}
  Zhu}]{Guo2015}
Guo J, Cheng J, Levina E, Michailidis G, Zhu J (2015) Estimating heterogeneous
  graphical models for discrete data with an application to roll call voting.
  \emph{The Annals of Applied Statistics} 9(2):821.

\bibitem[{H{\"a}nninen \protect\BIBand{} Kujala(2012)}]{Hanninen2012}
H{\"a}nninen M, Kujala P (2012) Influences of variables on ship collision
  probability in a {B}ayesian belief network model. \emph{Reliability
  Engineering \& System Safety} 102:27--40.

\bibitem[{Iooss \protect\BIBand{} Prieur(2019)}]{Iooss2019}
Iooss B, Prieur C (2019) Shapley effects for sensitivity analysis with
  correlated inputs: comparisons with {S}obol' indices, numerical estimation
  and applications. \emph{International Journal for Uncertainty Quantification}
  9(5).

\bibitem[{Kj{\ae}rulff(1990)}]{Kjaerulff:90}
Kj{\ae}rulff U (1990) Triangulation of graphs - algorithms giving small total
  state space. Technical Report R 90-09, Department of Mathematics and Computer
  Science, Strandvejen, DK 9000 Aalborg, Denmark.

\bibitem[{Kj{\ae}rulff \protect\BIBand{} van~der Gaag(2000)}]{Kjaerulff2000}
Kj{\ae}rulff U, van~der Gaag LC (2000) Making sensitivity analysis
  computationally efficient. \emph{Proceedings of the Sixteenth Conference on
  Uncertainty in Artificial Intelligence}, 317--325.

\bibitem[{Kleemann et~al.(2017)Kleemann, Celio, \protect\BIBand{}
  F{\"u}rst}]{Kleemann2017}
Kleemann J, Celio E, F{\"u}rst C (2017) Validation approaches of an
  expert-based {B}ayesian belief network in northern {G}hana, {W}est {A}frica.
  \emph{Ecological modelling} 365:10--29.

\bibitem[{Koller et~al.(2009)Koller, Friedman, \protect\BIBand{}
  Bach}]{Koller2009}
Koller D, Friedman N, Bach F (2009) \emph{Probabilistic graphical models:
  principles and techniques} (MIT press).

\bibitem[{L.~Salemi et~al.(2019)L.~Salemi, Song, Nelson, \protect\BIBand{}
  Staum}]{Salemi2019}
L~Salemi P, Song E, Nelson BL, Staum J (2019) Gaussian {M}arkov random fields
  for discrete optimization via simulation: framework and algorithms.
  \emph{Operations Research} 67(1):250--266.

\bibitem[{Lee \protect\BIBand{} Cichocki(2017)}]{LC:17}
Lee N, Cichocki A (2017) Fundamental tensor operations for large-scale data
  analysis using tensor network formats. \emph{Multidimensional Systems and
  Signal Processing} 1--40.

\bibitem[{Leonelli et~al.(2017)Leonelli, G\"{o}rgen, \protect\BIBand{}
  Smith}]{Leonelli2017}
Leonelli M, G\"{o}rgen C, Smith JQ (2017) Sensitivity analysis in multilinear
  probabilistic models. \emph{Information Sciences} 411:84--97.

\bibitem[{Leonelli et~al.(2021)Leonelli, Ramanathan, \protect\BIBand{}
  Wilkerson}]{Leonelli2021}
Leonelli M, Ramanathan R, Wilkerson RL (2021) {Sensitivity and robustness
  analysis in Bayesian networks with the bnmonitor R package}.
  \emph{arXiv:2107.11785} .

\bibitem[{Leonelli \protect\BIBand{} Riccomagno(2019)}]{Leonelli2019}
Leonelli M, Riccomagno E (2019) A geometric characterisation of sensitivity
  analysis in monomial models. \emph{arXiv:1901.02058} .

\bibitem[{Li \protect\BIBand{} Mahadevan(2017)}]{LM:17}
Li C, Mahadevan S (2017) Sensitivity analysis of a {B}ayesian network.
  \emph{ASCE-ASME Journal of Risk and Uncertainty in Engineering Systems, Part
  J: Mechanical Engineering} 4(1):011003.

\bibitem[{Luque \protect\BIBand{} D{\'\i}ez(2010)}]{Luque2010}
Luque M, D{\'\i}ez FJ (2010) Variable elimination for influence diagrams with
  super value nodes. \emph{International Journal of Approximate Reasoning}
  51(6):615--631.

\bibitem[{Ma et~al.(2019)Ma, Sudakov, Strong, \protect\BIBand{}
  Golden}]{Ma2019}
Ma YP, Sudakov I, Strong C, Golden KM (2019) Ising model for melt ponds on
  arctic sea ice. \emph{New Journal of Physics} 21(6):063029.

\bibitem[{Makaba et~al.(2020)Makaba, Doorsamy, \protect\BIBand{}
  Paul}]{Makaba2020}
Makaba T, Doorsamy W, Paul BS (2020) Bayesian network-based framework for
  cost-implication assessment of road traffic collisions. \emph{International
  Journal of Intelligent Transportation Systems Research} 19:240--253.

\bibitem[{Marrel et~al.(2009)Marrel, Iooss, Laurent, \protect\BIBand{}
  Roustant}]{MILR:09}
Marrel A, Iooss B, Laurent B, Roustant O (2009) Calculations of {Sobol} indices
  for the {Gaussian} process metamodel. \emph{Reliability Engineering \& System
  Safety} 94:742--751.

\bibitem[{McLachlan et~al.(2020)McLachlan, Dube, Hitman, Fenton,
  \protect\BIBand{} Kyrimi}]{Mclachlan2020}
McLachlan S, Dube K, Hitman GA, Fenton N, Kyrimi E (2020) Bayesian networks in
  healthcare: distribution by medical condition. \emph{Artificial Intelligence
  in Medicine} 101912.

\bibitem[{Nojavan et~al.(2017)Nojavan, Qian, \protect\BIBand{}
  Stow}]{Nojavan2017}
Nojavan F, Qian SS, Stow CA (2017) {Comparative analysis of discretization
  methods in Bayesian networks}. \emph{Environmental Modelling \& Software}
  87:64--71.

\bibitem[{Qazi et~al.(2016)Qazi, Quigley, Dickson, \protect\BIBand{}
  Kirytopoulos}]{QQDK:16}
Qazi A, Quigley J, Dickson A, Kirytopoulos K (2016) Project complexity and risk
  management ({ProCRiM}): Towards modelling project complexity driven risk
  paths in construction projects. \emph{International Journal of Project
  Management} 34(7):1183--1198.

\bibitem[{Rai(2014)}]{Rai:14}
Rai P (2014) \emph{Sparse Low Rank Approximation of Multivariate Functions --
  Applications in Uncertainty Quantification}. Doctoral thesis, {\'{E}cole
  Centrale Nantes},
  \urlprefix\url{https://tel.archives-ouvertes.fr/tel-01143694}.

\bibitem[{Razavi et~al.(2021)Razavi, Jakeman, Saltelli, Prieur, Iooss,
  Borgonovo, Plischke, Lo~Piano, Iwanaga, Becker et~al.}]{Razavi2021}
Razavi S, Jakeman A, Saltelli A, Prieur C, Iooss B, Borgonovo E, Plischke E,
  Lo~Piano S, Iwanaga T, Becker W, et~al. (2021) The future of sensitivity
  analysis: an essential discipline for systems modeling and policy support.
  \emph{Environmental Modelling \& Software} 137:104954.

\bibitem[{Robeva \protect\BIBand{} Seigal(2018)}]{RS:18}
Robeva E, Seigal A (2018) {Duality of graphical models and tensor networks}.
  \emph{Information and Inference: A Journal of the IMA} 8(2):273--288.

\bibitem[{Rohmer(2020)}]{Rohmer2020}
Rohmer J (2020) Uncertainties in conditional probability tables of discrete
  {B}ayesian belief networks: a comprehensive review. \emph{Engineering
  Applications of Artificial Intelligence} 88:103384.

\bibitem[{Saltelli et~al.(2019)Saltelli, Aleksankina, Becker, Fennell,
  Ferretti, Holst, Li, \protect\BIBand{} Wu}]{Saltelli2019}
Saltelli A, Aleksankina K, Becker W, Fennell P, Ferretti F, Holst N, Li S, Wu Q
  (2019) Why so many published sensitivity analyses are false: a systematic
  review of sensitivity analysis practices. \emph{Environmental Modelling \&
  Software} 114:29--39.

\bibitem[{Saltelli et~al.(2008)Saltelli, Ratto, Andres, Campolongo, Cariboni,
  Gatelli, Saisana, \protect\BIBand{} Tarantola}]{SRACCGST:08}
Saltelli A, Ratto M, Andres T, Campolongo F, Cariboni J, Gatelli D, Saisana M,
  Tarantola S (2008) \emph{Global Sensitivity Analysis: The Primer} (John Wiley
  \& Sons, Ltd.).

\bibitem[{Saltelli et~al.(2000)Saltelli, Tarantola, \protect\BIBand{}
  Campolongo}]{Saltelli2000}
Saltelli A, Tarantola S, Campolongo F (2000) Sensitivity anaysis as an
  ingredient of modeling. \emph{Statistical Science} 15(4):377--395.

\bibitem[{Scutari(2010)}]{Scutari2010}
Scutari M (2010) Learning {B}ayesian networks with the bnlearn {R} package.
  \emph{Journal of Statistical Software} 35(3):1--22.

\bibitem[{Sheikholeslami et~al.(2020)Sheikholeslami, Gharari, Papalexiou,
  \protect\BIBand{} Clark}]{Sh2020}
Sheikholeslami R, Gharari S, Papalexiou SM, Clark MP (2020) Viscous: A
  variance-based sensitivity analysis using copulas for efficient
  identification of dominant hydrological processes. \emph{Earth and Space
  Science Open Archive} 37.

\bibitem[{Sobol(1990)}]{Sobol:90}
Sobol IM (1990) Sensitivity estimates for nonlinear mathematical models (in
  {R}ussian). \emph{Mathematical Models} 2:112--118.

\bibitem[{Stoudenmire \protect\BIBand{} Schwab(2016)}]{SS:16}
Stoudenmire E, Schwab DJ (2016) Supervised learning with tensor networks. Lee
  D, Sugiyama M, Luxburg U, Guyon I, Garnett R, eds., \emph{Advances in Neural
  Information Processing Systems}, volume~29.

\bibitem[{Sudret(2008)}]{Sudret:08}
Sudret B (2008) Global sensitivity analysis using polynomial chaos expansions.
  \emph{Reliability Engineering \& System Safety} 93(7):964 -- 979.

\bibitem[{{Ye} \protect\BIBand{} {Lim}(2018)}]{YL:18}
{Ye} K, {Lim} LH (2018) {Tensor network ranks}. \emph{arXiv:1801.02662} .

\bibitem[{Zhang et~al.(2015)Zhang, Yang, Oseledets, Karniadakis,
  \protect\BIBand{} Daniel}]{ZYOKD:15}
Zhang Z, Yang X, Oseledets IV, Karniadakis GE, Daniel L (2015) Enabling
  high-dimensional hierarchical uncertainty quantification by {ANOVA} and
  tensor-train decomposition. \emph{IEEE Transactions on Computer-Aided Design
  of Integrated Circuits and Systems} 34(1):63--76.

\end{thebibliography}

\end{document}